\documentclass[12pt]{article}

\ifdefined\usebigfont
\usepackage{neurips_2019_jl}

\usepackage{times}
\usepackage[fontsize=13pt]{scrextend}
\usepackage[left=1.56in,right=1.56in,top=1.74in,bottom=1.74in]{geometry}
\else
\usepackage{times}
\usepackage[]{geometry}
\fi
\usepackage[utf8]{inputenc} % allow utf-8 input
\usepackage[T1]{fontenc}    % use 8-bit T1 fonts
\usepackage{hyperref}       % hyperlinks
\usepackage{url}            % simple URL typesetting
\usepackage{booktabs}       % professional-quality tables
\usepackage{amsfonts}       % blackboard math symbols
\usepackage{nicefrac}       % compact symbols for 1/2, etc.
\usepackage{microtype}      % microtypography
\usepackage{amsthm}
\usepackage{amssymb}
\usepackage{amsmath}
\usepackage{comment}
\usepackage{mathrsfs}

\newcommand{\bracket}[1]{\left(#1\right)}

\DeclareMathOperator*{\argmax}{argmax}

\newcommand{\bbR}{\mathbb{R}}
\newcommand{\bbE}{\mathbb{E}}

\newcommand{\bbP}{\mathbb{P}}

\newcommand{\bfw}{\mathbf{w}}
\newcommand{\bfW}{\mathbf{W}}
\newcommand{\norm}[1]{\left\lVert#1\right\rVert}
\newcommand{\rfnorm}[1]{\left\lVert#1\right\rVert_{\text{RF}}}
\newcommand{\rkhs}[1]{\left\lVert#1\right\rVert_{\cH}}

\newcommand{\rf}{\mathcal{F}_{\text{RF}}}
\newcommand{\I}{\mathbb{I}}
\newcommand{\bfc}{\mathbf{c}}

\newcommand{\bfx}{\mathbf{x}}
\newcommand{\bfD}{\mathbf{D}}

\newcommand{\bfA}{\mathbf{A}}

\newcommand{\bfI}{\mathbf{I}}

\newcommand{\bfa}{\mathbf{a}}

\newcommand{\vc}{\operatorname{vec}}

\def\R{\mathbb{R}}

\def\cA{\mathcal{A}}
\def\cB{\mathcal{B}}
\def\cD{\mathcal{D}}

\def\cF{\mathcal{F}}

\def\cH{\mathcal{H}}

\def\cN{\mathcal{N}}
\def\cP{\mathcal{P}}
\def\cS{\mathcal{S}}

\def\vw{\mathbf{w}}
\def\vW{\mathbf{W}}
\def\vV{\mathbf{V}}

\def\vx{\mathbf{x}}
\def\vy{\mathbf{y}}

\def\ntk{K_{\sigma}}

\newcommand{\mat}[1]{\mathbf{#1}}
\newcommand{\vect}[1]{\mathbf{#1}}

\newcommand{\abs}[1]{\left|#1\right|}

\newcommand{\relu}[1]{\sigma\left(#1\right)}
\newcommand{\ReLU}[1]{\mathrm{ReLU}\left(#1\right)}

\newtheorem{thm}{Theorem}[section]
\newtheorem{lem}{Lemma}[section]
\newtheorem{cor}{Corollary}[section]
\newtheorem{prop}{Proposition}[section]
\newtheorem{asmp}{Assumption}[section]
\newtheorem{defn}{Definition}[section]

\newtheorem{rem}{Remark}[section]

\usepackage{xcolor}

\usepackage{xr}
\makeatletter
\newcommand*{\addFileDependency}[1]{% argument=file name and extension
  \typeout{(#1)}
  \@addtofilelist{#1}
  \IfFileExists{#1}{}{\typeout{No file #1.}}
}
\makeatother

\newtheorem*{remark}{Remark}
\title{Convergence of Adversarial Training in Overparametrized Neural Networks}

\author{%
    Ruiqi Gao$^{1,}$\thanks{Joint first author.}, Tianle Cai$^{1,}$\footnotemark[1], Haochuan Li$^2$, Liwei Wang$^1$, Cho-Jui Hsieh$^3$,\\and Jason D.~Lee$^4$\\\\
    $^1$Peking University\\
    $^2$Massachusetts Institute of Technology\\
    $^3$University of California, Los Angeles\\
    $^4$Princeton University
}
\begin{document}

\maketitle
% \begin{abstract}
    Neural networks are vulnerable to adversarial examples, i.e. inputs that are imperceptibly perturbed from natural data and yet incorrectly classified by the network. Adversarial training \cite{madry2017towards}, a heuristic form of robust optimization that alternates between minimization and maximization steps, has proven to be among the most successful methods to train networks to be robust against a pre-defined family of perturbations. This paper provides a partial answer to the success of adversarial training, by showing that it converges to a network where the surrogate loss with respect to the the attack algorithm is within $\epsilon$ of the optimal robust loss. Then we show that the optimal robust loss is also close to zero, hence adversarial training finds a robust classifier. The analysis technique leverages recent work on the analysis of neural networks via Neural Tangent Kernel (NTK), combined with motivation from online-learning when the maximization is solved by a heuristic, and the expressiveness of the NTK kernel in the $\ell_\infty$-norm. In addition, we also prove that robust interpolation requires more model capacity, supporting the evidence that adversarial training requires wider networks.
% \end{abstract}

\section{Introduction}
\label{sec:intro}

Recent studies have demonstrated that neural network models, despite achieving human-level performance on many important tasks, 
are not robust to adversarial examples---a small and human imperceptible input perturbation can easily change the prediction label~\cite{szegedy2013intriguing,goodfellow2015explaining}.  
This phenomenon brings out security concerns when deploying neural network models to real world systems~\cite{eykholt2017robust}. 
In the past few years, many defense algorithms
have been developed~\cite{guo2017countering,song2017pixeldefend,ma2018characterizing,liu2018towards,
samangouei2018defense} 
 to improve the network's robustness, but most of them are still vulnerable under stronger attacks, as reported in~\cite{athalye2018obfuscated}. 
Among current defense methods, adversarial training~\cite{madry2017towards} has become one of the most successful methods to train robust neural networks. 

To obtain a robust network, we need to consider the ``robust loss'' instead of a regular loss. 
The robust loss is defined as the maximal loss within a neighborhood around the input of each sample, 
and minimizing the robust loss under empirical distribution leads to a min-max optimization problem. 
Adversarial training~\cite{madry2017towards} is a way to minimize the robust loss. At each iteration, it (approximately) solves the inner maximization
problem by an attack algorithm $\cA$ to get an adversarial sample, and then runs a (stochastic) gradient-descent update to minimize the loss on the adversarial samples. 
Although adversarial training has been widely used in practice and hugely improves the robustness of neural networks in many applications, its convergence properties are still unknown. It is unclear whether a network with small robust error exists and whether adversarial training is able to converge to a solution with minimal adversarial train loss. 

In this paper, we study the convergence of adversarial training algorithms and try to answer the above questions on over-parameterized neural networks. 
%In this paper, we analyze the performance of adversarial training. 
We consider width-$m$ neural networks both for the setting of deep networks with $H$ layers, and two-layer networks for some additional analysis. % smooth activation, and $n$ training samples. This assumption holds for many activation functions including the soft-plus and sigmoid.
Our contributions are summarized below.
\begin{itemize}
	\item For an $H$-layer deep network with ReLU activations, and an arbitrary attack algorithm, when the width $m$ is large enough, we show that projected gradient descent converges to a network where the surrogate loss with respect to the attack $\cA$ is within $\epsilon$ of the optimal robust loss (Theorem \ref{thm:converge_deep}). The required width is \emph{polynomial} in the depth and the input dimension.
	
	\item For a two-layer network with smooth activations, we provide a proof of convergence, where the projection step is not required in the algorithm (Theorem \ref{thm:converge_two_layer_projless}).
	
	\item We then consider the expressivity of neural networks w.r.t. robust loss (or robust interpolation). We show when the width $m$ is sufficiently large, the neural network can achieve optimal robust loss $\epsilon$; see Theorems~\ref{thm:robust_exist} and \ref{thm:qurelu_approx} for the precise statement. By combining the expressivity result and the previous bound of the loss over the optimal robust loss, we show that adversarial training finds networks of small robust training loss (Corollary \ref{cor:exist-robust} and Corollary \ref{cor:exist-robust-qurelu}).

	%We show even a two layer network can find a robust classifier that has error $\epsilon$ with a data dependent RF-norm $B_{\epsilon}$. We then show when the width $m=\text{poly}(B_{\epsilon})$, the neural network can achieve optimal robust loss $\epsilon$; see Theorem~\ref{thm:robust_exist} for precise statement.
    %\tnote{Modify this part to -----------}
	\item We show that the VC-Dimension of the model class which can \emph{robustly} interpolate any $n$ samples is lower bounded by $\Omega(nd)$ where $d$ is the dimension. In contrast, there are neural net architectures that can interpolate $n$ samples with only $O(n)$ parameters and VC-Dimension at most $O(n\log n)$. Therefore, the capacity required for robust learning is higher. %Thus robust learning provably requires larger capacity.
\end{itemize}

\section{Related Work}
\label{sec:related}
\paragraph{Attack and Defense}
Adversarial examples are inputs that are slightly perturbed from a natural sample and yet incorrectly classified by the model.
An adversarial example can be generated by maximizing the loss function within an $\epsilon$-ball around a natural sample.
Thus, generating adversarial examples can be viewed as solving a constrained optimization problem and can be (approximately) solved by a projected gradient descent (PGD) method~\cite{madry2017towards}. Some other techniques have also been proposed in the literature including L-BFGS~\cite{szegedy2013intriguing}, FGSM~\cite{goodfellow2015explaining}, iterative FGSM~\cite{kurakin2016adversarial} and C\&W attack~\cite{carlini2017towards}, where they differ from each other by the distance measurements, loss function or optimization algorithms. There are also studies on adversarial attacks with limited information about the target model. For instance, \cite{chen2017zoo,ilyas2018black,brendel2017decision} considered the black-box setting where the model is hidden but the attacker can make queries and get the corresponding outputs of the model.
%output of the model; \cite{} considered the decision-based black box setting where only hard-label prediction can be observed.

Improving the robustness of neural networks against adversarial attacks, also known as defense, has been recognized as an important and unsolved problem in machine learning. Various kinds of defense methods have been proposed~\cite{guo2017countering,song2017pixeldefend,ma2018characterizing,liu2018towards,
samangouei2018defense}, but many of them are based on obfuscated gradients which does not really improve robustness under stronger attacks~\cite{athalye2018obfuscated}. As an exception, \cite{athalye2018obfuscated} reported that the adversarial training method developed in~\cite{madry2017towards} is the only defense that works even under carefully designed attacks.

\paragraph{Adversarial Training}
Adversarial training is one of the first defense ideas proposed in earlier papers~\cite{goodfellow2015explaining}. The main idea is to add adversarial examples into the training set to improve the robustness. However, earlier work usually only adds adversarial example once or only few times during the training phase.
Recently, \cite{madry2017towards} showed that adversarial training can be viewed as solving a min-max optimization problem where the training algorithm aims to minimize the robust loss, defined as the maximal loss within a certain $\epsilon$-ball around each training sample.
Based on this formulation, a clean adversarial training procedure based on PGD-attack has been developed and achieved state-of-the-art results even under strong attacks.
This also motivates some recent research on gaining theoretical understanding
of robust error~\cite{bubeck2018adversarial,schmidt2018adversarially}. Also,
adversarial training suffers from slow training time since it runs several
steps of attacks within one update, and several recent works are trying to
resolve this issue \cite{shafahi2019adversarial,zhang2019you}. From the
theoretical perspective, a recent work \cite{wang2019convergence} considers to
quantitatively evaluate the convergence quality of adversarial examples found
in the inner maximization and therefore ensure robustness.
\cite{yin2018rademacher} consider generalization upper and lower bounds for
robust generalization. \cite{liu2018adversarial} improves the robust
generalization by data augmentation with GAN. \cite{gonen2018learning}
considers to reduce the optimization of min-max problem to online learning
setting and use their results to analyze the convergence of GAN.
In this paper, our analysis for adversarial is quite general and is not restricted to any specific kind of attack algorithm.
%is independent to the attacking algorithm (denoted by $\cA$) so is quite general f
%an attacking algorithm $\cA$ is trying to maximize the loss within some $\epsilon$-ball and the defender aims to minimize the ne
%\tnote{We haven't worked on certified robustness so maybe this part can be shorten and moved to discussion? (The convergence of random smoothing is very simple though)}

\paragraph{Global convergence of Gradient Descent} Recent works on the over-parametrization of neural networks prove that when the width greatly exceeds the sample size, gradient descent converges to a global minimizer from random initialization \cite{li2018learning,du2018gradient_deep,du2018gradient,allen2018convergencetheory,zou2018stochastic}. The key idea in the earlier literature is to show that the Jacobian w.r.t. parameters has minimum singular value lower bounded, and thus there is a global minimum near every random initialization, with high probability. However for the robust loss, the maximization cannot be evaluated and the Jacobian is not necessarily full rank. For the surrogate loss, the heuristic attack algorithm may not even be continuous and so the same arguments cannot be utilized.

\paragraph{Certified Defense and Robustness Verification}
%For each sample, the robust loss is defined as the maximum loss within a local neighborhood.
% ($\max_{x' \in B(x)}\text{loss}(f(W, x_i), y_i)$),
%Due to the non-convexity, attack algorithms usually fail to find the exact
%maximum, so robust error computed by an attack algorithm cannot give us a formal guarantee of robustness.
%to formally guarantee the robustness
%of a model.
%As a consequence, networks trained by standard adversarial training algorithms~\cite{madry2017towards}, although being robust under strong attacks, do not have a certified guarantee of robustness.
In contrast to attack algorithms, neural network verification methods \cite{wong2018provable, weng2018towards,zhang2018efficient, singh2018fast, cohen2019certified, salman2019convex} tries to find upper bounds of the robust loss and provide certified robustness measurements.
%Several algorithms have been proposed recently. \cite{wong2018provable} proposed to solve the dual of a linear relaxation problem to obtain a certified bound. \cite{weng2018towards,zhang2018efficient} provides a similar algorithm based on primal relaxation. \cite{singh2018fast} proposed another approach based on abstract interpretation. More recently, \cite{salman2019convex} provided a unified view, showing that most of the existing verification methods are based on a convex relaxation of ReLU network.
Equipped with these verification methods for computing upper bounds of robust error, one can then apply adversarial training to get a network with certified robustness. Our analysis in Section \ref{sec:main_opt} can also be extended to certified adversarial training.
%This is first proposed in~\cite{wong2018provable}.
%At each iteration, instead of finding a lower bound of robust error by attack, we can find an upper bound of robust error by verification and and train the model to minimize this upper bound. Several certified adversarial training algorithms along this line have been proposed recently~\cite{wong2018scaling,dvijotham2018training}. Our analysis in Section \ref{sec:main_opt} can also incorporate certified adversarial training.

\section{Preliminaries}
\label{sec:pre}
\subsection{Notations}
Let $[n] = \{1, 2, \ldots, n\}$.  
We use $\cN(\vect{0},\mat{I})$ to denote the standard Gaussian distribution. 
%We use $\cS = \{\vx: \norm{\vx}_2 = 1\}$ to denote the surface of the unit Euclidean ball. For a function $f:\R^d\to\R$ and a set $S\subset\R^d$, we let $\norm{f}_{\infty, S} = \sup_{\vx\in S}\abs{f(\vx)}$. 
%For lattice vector $\alpha = (\alpha_j: j\in[d])\in\Z_+^d$, we denote the monomial $x^\alpha = x_1^{\alpha_1}\cdots x_d^{\alpha_d}$.
%For a matrix $\vect A$, we use $\mat A_{ij}$ to denote its $(i, j)$-th entry. We will also use $\mat{A}_{i,:}$ to denote the $i$-th row vector of $\mat{A}$ and define $\mat{A}_{i,j:k}=(\mat{A}_{i,j},\mat{A}_{i,j+1},\cdots,\mat{A}_{i,k})$ as part of the vector. 
%Similarly $\mat{A}_{:,i}$ is the $i$-th column vector and $\mat{A}_{j:k,i}$ is a part of $i$-th column vector.
For a vector $\vect{v}$, we use $\norm{\vect{v}}_2$ to denote the Euclidean norm.
For a matrix $\mat{A}$ we use $\norm{\mat{A}}_F$ to denote the Frobenius norm and $\norm{\mat{A}}_2$ to denote the spectral norm.
% If a matrix $\mat{A}$ is positive semi-definite, we use $\lambda_{\min}(\mat{A})$ to denote its smallest eigenvalue.
We use $\langle \cdot, \cdot \rangle$ to denote the standard Euclidean inner product between two vectors, matrices, or tensors.
We let $O(\cdot)$, $\Theta(\cdot)$ and $\Omega\left(\cdot\right)$ denote standard Big-O, Big-Theta and Big-Omega notations that suppress multiplicative constants. %We also use $\tilde{O}(\cdot)$, $\tilde{\Theta}(\cdot)$ and $\tilde{\Omega}\left(\cdot\right)$ to denote these notations which suppress $\log$ factors.
%In this paper we will use  $C$ and $c$ to denote constants. The specific value can be different from line to line.

\subsection{Deep Neural Networks}
Here we give the definition of our deep fully-connected neural networks. For the convenience of proof, we use the same architecture as defined in \cite{allen2018convergencetheory}.\footnote{We only consider the setting when the network output is scalar. However, it is not hard to extend out results to the setting of vector outputs.} Formally, we consider a neural network of the following form.

Let $\vx \in \mathbb{R}^{d}$ be the input, the fully-connected neural network is defined as follows: $\bfA \in \mathbb{R}^{m \times d}$ is the first weight matrix, $\bfW^{(h)} \in \mathbb{R}^{m \times m}$ is the weight matrix at the $h$-th layer for $h\in[H]$, $\bfa \in \mathbb{R}^{m\times 1}$ is the output layer, and $\sigma(\cdot)$ is the ReLU activation function.\footnote{We assume intermediate layers are square matrices of size $m$ for simplicity. It is not difficult to generalize our analysis to rectangular weight matrices.}  The parameters are $\bfW = (\vc\{\bfA\}^\top,\vc\{\bfW^{(1)}\}^\top, \cdots, \vc\{\bfW^{(H)}\}^\top, \bfa^\top)^\top$. However, without loss of generality, during training we will fix $\bfA$ and $\bfa$ once initialized, so later we will refer to $\bfW$ as $\bfW = (\vc\{\bfW^{(1)}\}^\top, \cdots, \vc\{\bfW^{(H)}\}^\top)^\top$. The prediction function is defined recursively:

\begin{align}
\vx^{(0)}&=\bfA \vx\nonumber\\
\overline{\vx}^{(h)}&={\bfW}^{(h)}\vx^{(h-1)},\quad h\in[H]\label{equ:network_form}\\
\vx^{(h)}&=\sigma\left(\overline{\vx}^{(h)}\right),\quad h\in[H]\nonumber\\
f(\bfW,\vx)&=\bfa^\top \vx^{(H)}, \nonumber
\end{align}
where $\overline{\vx}^{(h)}$ and $\vx^{(h)}$ are the feature vectors before and after the activation function, respectively. Sometimes we also denote $\overline{\vx}^{(0)}=\vx^{(0)}$.

% \begin{align}
%     \label{equ:network_form}
%     f(\vW, \vx) = \frac{1}{\sqrt{m}}\sum_{r=1}^ma_r\sigma(\vw_r^\top \vx)
% \end{align}
% where $\vx\in\R^d$ is the input, $\vw_r\in\R^d$ is the weight vector of the first layer, $a_r\in\R$ is the output weight and $\relu{\cdot}$ is the activation function. 
 
\begin{comment}
\begin{asmp}[Smoothness of activation function]\label{asmpt:smoothness}
The activation function is Lipschitz and smooth, that is, we can assume there exists a constant $C>0$ such that for any $z\in\bbR$
\begin{align*}
\abs{\sigma'(z)}\le C \text{ and }\sigma'(z)\text{ is }C\text{-Lipschitz}.
\end{align*}
\end{asmp}
\end{comment}

We use the following initialization scheme: Each entry in $\bfA$ and $\bfW^{(h)}$ for $h \in [H]$ follows the i.i.d. Gaussian distribution $\mathcal{N}(0, \frac{2}{m})$, and each entry in $\bfa$ follows the i.i.d. Gaussian distribution $\cN(0,1)$. As we mentioned, we only train on $\vW^{(h)}$ for $h \in [H]$ and fix $\bfa$ and $\bfA$.  For a training set $\{\vx_i, y_i\}_{i=1}^n$, the loss function is denoted $\ell: (\mathbb{R}, \mathbb{R})\mapsto \mathbb{R}$, and the (non-robust) training loss is
$L(\vW) = \frac{1}{n}\sum_{i=1}^n \ell(f(\vW, \vx_i),y_i).$ We make the following assumption on the loss function:

\begin{asmp}[Assumption on the Loss Function]\label{asmpt:smoothness-loss}
The loss  $\ell(f(\bfW, \vx),y)$ is Lipschitz, smooth, convex in $f(\bfW, \vx)$ and satisfies $\ell(y,y) = 0$.
\end{asmp}
%These two conditions will be used to guarantee the smoothness of the function represented by the neural network.

% To learn the neural network, we use the use the randomly initialized gradient descent method. 

% The key architectural parameter is the width $m$. As we shall see, in general, in order to obtain better train loss we need  width $m$, and so for overparametrized networks we are able to minimize the robust train loss.

\subsection{Perturbation and the Surrogate Loss Function}
The goal of adversarial training is to make the model robust in a neighbor of each datum. We first introduce the definition of the perturbation set function to determine the perturbation at each point.
\begin{defn}[Perturbation Set]\label{def:perturb_set}
Let the input space be $\mathcal{X}\subset\mathbb{R}^d$. The perturbation set function is $\cB: \mathcal{X}\to \cP(\mathcal{X})$, where $\cP(\mathcal{X})$ is the power set of $\mathcal{X}$. At each data point $\vx$, $\cB(\vx)$ gives the perturbation set on which we would like to guarantee robustness. For example, a commonly used perturbation set is $\cB(\vx) = \{\vx' : \|\vx'-\vx\|_2\le\delta\}$. Given a dataset $\{\vx_i, y_i\}_{i=1}^{n}$, we say that the perturbation set is compatible with the dataset if $\overline{\mathcal{B}(\vx_i)}\cap\overline{\mathcal{B}(\vx_j)} \neq \phi$ implies $y_i = y_j$. In the rest of the paper, we will always assume that $\mathcal{B}$ is compatible with the given data.
\end{defn}

 Given a perturbation set, we are now ready to define the perturbation function that maps a data point to another point inside its perturbation set. We note that the perturbation function can be quite general including the identity function and any adversarial attack\footnote{It is also not hard to extend our analysis to perturbation functions involving randomness.}. Formally, we give the following definition.
 
\begin{defn}[Perturbation Function]\label{def:perturb_func}
A perturbation function is defined as a function $\cA:\mathcal{W}\times\R^d\to\R^d$, where $\mathcal{W}$ is the parameter space. Given the parameter $\vW$ of the neural network~\eqref{equ:network_form}, $\cA(\vW, \vx)$ maps $\vx\in\R^d$ to some $\vx'\in\cB(\vx)$ where $\cB(\vx)$ refers to the perturbation set defined in Definition~\ref{def:perturb_set}.
\end{defn}

Without loss of generality, throughout Section 4 and 5, we will restrict our input $\bfx$ as well as the perturbation set $\cB(\bfx)$ within the surface of the unit ball $\cS = \{\bfx\in\bbR^{d}: \norm{\bfx}_2 = 1\}$. 

With the definition of perturbation function, we can now define a large family of loss functions on the training set $\{\vx_i, y_i\}_{i=1}^n$. We will show this definition covers the standard loss used in empirical risk minimization and the robust loss used in adversarial training.

\begin{defn}[Surrogate Loss Function]\label{def:loss}
Given a perturbation function $\cA$ defined in Definition~\ref{def:perturb_func}, the current parameter $\vW$ of a neural network $f$, and a training set $\{\vx_i, y_i\}_{i=1}^n$, we define the surrogate loss $L_\cA(\vW)$ on the training set as
\begin{align*}
    L_\cA(\vW) = \frac{1}{n}\sum_{i = 1}^n \ell(f(\vW, \cA(\vW, \vx_i)), y_i).
\end{align*}
\end{defn}

It can be easily observed that the standard training loss $L(\vW)$ is a special case of surrogate loss function when $\cA$ is the identity. The goal of adversarial training is to minimize the \emph{robust loss}, i.e. the surrogate loss when $\cA$ is the strongest possible attack. The formal definition is as follows:

\begin{defn}[Robust Loss Function] The robust loss function is defined as
\begin{align}\nonumber
    L_\ast(\vW) := L_{\cA^\ast}(\vW)
\end{align}
where
\begin{align}\nonumber
    \cA^\ast(\vW, \vx_i) = \argmax_{\vx_i'\in\cB(\vx_i)} \ell(f(\vW, \vx_i'), y_i).
\end{align}
\end{defn}

\section{Convergence Results of Adversarial Training}
\label{sec:convergence}
We consider optimizing the surrogate loss $L_\cA$ with the perturbation function $\cA(\vW, \vx)$ defined in Definition~\ref{def:perturb_func}, which is what adversarial training does given any attack algorithm $\mathcal{A}$. In this section, we will prove that for a neural network with sufficient width, starting from the initialization $\bfW_0$, after certain steps of projected gradient descent within a convex set $B(R)$, the loss $L_\cA$ is provably upper-bounded by the best minimax robust loss in this set
\begin{align}\nonumber
\min_{\vW \in B(R)}L_*(\bfW),
\end{align}
where
\begin{align}
    B(R)=\left\{\vW:\norm{\bfW^{(h)}-\bfW^{(h)}_0}_F\le \frac{R}{\sqrt{m}},h\in[H]\right\}\label{equ:rwb}.
\end{align}
% where $R$ depends polynomially on the smoothness parameters of Assumptions~\ref{asmpt:smoothness-loss}.

\label{sec:main_opt}
Denote $\cP_{B(R)}$ as the Euclidean projection to the convex set $B(R)$. Denote the parameter $\vW$ after the $t$-th iteration as $\vW_{t}$, and similarly $\vW^{(h)}_t$. For each step in adversarial training, projected gradient descent takes an update
\begin{align*}
\vV_{t+1} &= \vW_t - \alpha\nabla_\vW L_\cA (\vW_t),\\
\vW_{t+1} &=\cP_{B(R)}(\vV_{t+1}),
\end{align*}
where 
$$\nabla_\vW L_\cA (\vW) = \frac{1}{n}\sum_{i=1}^n l'\left(f(\vW, \cA(\vW, \vx_i)),y_i\right) \nabla_\vW f(\vW, \cA(\vW, \vx_i)),$$
and the derivative $\ell'$ stands for $\frac{\partial\ell}{\partial f}$, \emph{the gradient $\nabla_\vW f$ is with respect to the first argument $\vW$}.
%and the Euclidean projection $\vW=\cP_{\cR(\vW_0,B)}(\vV)$ has a simple closed form as follows:
% \begin{equation}
% w_r^{(h)}=\left\{
% \begin{array}{lr}
% v_r^{(h)},&\text{ if } \norm{v_r^{(h)}-w_r^{(h)}(0)}_2 \le \frac{B}{\sqrt{m}},\\
% w_r^{(h)}(0)+\frac{B}{\sqrt{m}} \cdot\frac{v_r^{(h)}-w_r^{(h)}(0)}{\norm{v_r^{(h)}-w_r^{(h)}(0)}_2},&\text{ if } \norm{v_r^{(h)}-w_r^{(h)}(0)}_2 > \frac{B}{\sqrt{m}},
% \end{array}
% \right.
% \end{equation}
% which is easy to implement in practice.

Specifically, we have the following theorem. 
\begin{thm}[Convergence of Projected Gradient Descent for Optimizing Surrogate Loss]\label{thm:converge_deep}
	 Given $\epsilon > 0$, suppose $R = \Omega(1)$, and $m\ge\max\bracket{\Theta\left(\frac{R^9H^{16}}{\epsilon^7}\right), \Theta(d^{2})}$. Let the loss function satisfy Assumption~\ref{asmpt:smoothness-loss}.\footnote{We actually didn't use the assumption $\ell(y, y) = 0$ in the proof, so common loss functions like the cross-entropy loss works in this theorem. Also, with some slight modifications, it is possible to prove for other loss functions including the square loss.} If we run projected gradient descent based on the convex constraint set $B(R)$ with stepsize $\alpha= O\left(\frac{\epsilon}{mH^2}\right)$ for $T=\Theta\bracket{\frac{R^2}{m\epsilon\alpha}} = \Omega\bracket{\frac{R^2H^2}{\epsilon^2}}$ steps, then with high probability we have
	\begin{equation}
	\label{equ:converge_main}
	\min_{t=1, \cdots, T}L_\cA(\bfW_t) - L_\ast(\bfW_{\ast}) \leq \epsilon,
	\end{equation}
	where $\bfW_\ast = \arg \min_{\bfW\in B(R)}L_\ast(\bfW)$.
\end{thm}

\begin{remark} 
Recall that $L_{\cA} (\bfW)$ is the loss suffered with respect to the perturbation function $\cA$. This means, for example, if the adversary uses the projected gradient ascent algorithm, then the theorem guarantees that projected gradient ascent cannot successfully attack the learned network. The stronger the attack algorithm is during training, the stronger the guaranteed surrogate loss becomes.
\end{remark}
\begin{remark}
 The value of $R$ depends on the approximation capability of the network, i.e. the greater $R$ is, the less $L_\ast(\bfW_{\ast})$ will be, thus affecting the overall bound on $\min_tL_\cA(\bfW_t)$. We will elaborate on this in the next section, where we show that for $R$ independent of $m$ there exists a network of small adversarial training error.
\end{remark}
	
% \begin{remark}In fact, there is a trade-off on the possible values of $R$, $\delta$, and $m$. The value of $R$ depends on the approximation capabilities of the network, i.e. the greater $R$ is, the less $L_\ast(\bfW_{\ast})$ will be, thus affecting the overall bound on $L_\cA(\bfW_t)$, as quantified later in the next section. And the smaller $R$ is, the less $M$ is required and more $\delta$ is allowed. In the end as long as (1)$R+\delta = O\bracket{\frac{1}{H^6(\log m)^{3/2}}}$ and (2)$R(R+\delta)^{1/3}H^{5/2}\sqrt{m\log m} \le \epsilon$, the proposition always hold, as can be seen in the full proof in Appendix~\ref{sec:deepopt}.
% For two-layer networks $H=1$, the update on $\vW$ does not require the projection step as it is implicitly enforced by gradient descent.\haochuan{How do we deal with two-layer result now? For deep net, we still use projected gd. However, it is just for simplicity not necessary.}
% \end{remark}

\subsection{Proof Sketch}
Our proof idea utilizes the same high-level intuition as 
\cite{allen2018convergencetheory,li2018learning,du2018gradient_deep,zou2018stochastic,cai2019gram,cao2019generalization} that near the initialization the network is linear. However, unlike these earlier works, the surrogate loss neither smooth, nor semi-smooth so there is no Polyak gradient domination phenomenon to allow for the global geometric contraction of gradient descent. In fact due to the the generality of perturbation function $\cA$ allowed, the surrogate loss is not differentiable or even continuous in $\vW$, and so the standard analysis cannot be applied. Our analysis utilizes two key observations. First the network $f(\vW, \cA(\vW,\vx))$ is still smooth w.r.t. the first argument\footnote{It is not jointly smooth in $\bfW$, which is part of the subtlety of the analysis.}, and is close to linear in the first argument near initialization, which is shown by directly bounding the Hessian w.r.t. $\vW$. Second, the perturbation function $\cA$ can be treated as an adversary providing a worst-case loss function $\ell_{\cA}(f,y)$ as done in online learning. However, online learning typically assumes the sequence of losses is convex, which is not the case here. We make a careful decoupling of the contribution to non-convexity from the first argument and the worst-case contribution from the perturbation function, and then we can prove that gradient descent succeeds in minimizing the surrogate loss. The full proof is in Appendix~\ref{sec:deepopt}.

\section{Adversarial Training Finds Robust Classifier}
\label{sec:adv_robust}
Motivated by the optimization result in Theorem~\ref{thm:converge_deep}, we hope to show that there is indeed a robust classifier in $B(R)$. To show this, we utilize the connection between neural networks and their induced Reproducing Kernel Hilbert Space (RKHS) via viewing networks near initialization as a random feature scheme \cite{daniely2017sgd,daniely2016toward,jacot2018neural,arora2019fine}. Since we only need to show the existence of a network architecture that robustly fits the training data in $B(R)$ and neural networks are at least as expressive as their induced kernels, we may prove this via the RKHS connection. The strategy is to first show the existence of a robust classifier in the RKHS, and then show that a sufficiently wide network can approximate the kernel via random feature analysis. The approximation results of this section will be, in general, exponential in dimension dependence due to the known issue of $d$-dimensional functions having exponentially large RKHS norm \cite{bach2017breaking}, so only offer \textit{qualitative guidance on existence of robust classifiers}.

Since deep networks contain two-layer networks as a sub-network, and we are concerned with expressivity, we focus on the local expressivity of two-layer networks. We write the standard two-layer network in the suggestive way\footnote{This makes $f(\bfW, \bfx) = 0$ at initialization, which helps eliminate some unnecessary technical nuisance.} (where the width $m$ is an even number)
\begin{align}\label{def:two_layer_zero_init_output}
    f(\vW, \vx) = \frac{1}{\sqrt{m}}\left(\sum_{r=1}^{m/2} a_r\sigma(\vw_r^\top \vx) + \sum_{r=1}^{m/2} a'_r\sigma(\bar{\vw}_r^\top \vx)\right),
\end{align}
and initialize as $\vw_r\sim\cN(0, \mathbf{I}_d)$ i.i.d. for $r = 1, \cdots, \frac{m}{2}$, and $\bar{\vw}_r$ is set to be equal to $\vw_r$, $a_r$ is randomly drawn from $\{1, -1\}$ and $a'_r=-a_r$. %We denote $\vw_{r0}, a_{r0}, \bar{\vw}_{r0}, a'_{r0}$ the initialization parameters respectively.
Similarly, we define the set
$B(R)=\left\{\vW:\norm{\bfW-\bfW_0}_F\le R \right\}$\footnote{Note that we have taken out the term $\frac{1}{\sqrt{m}}$ explicitly in the network expression for convenience, so in this section there is a difference of scaling by a factor of $\sqrt{m}$ from the $\bfW$ used in the previous section.} for $\bfW = (\bfw_1, \cdots, \bfw_{m/2}, \bar{\vw}_1, \cdots, \bar{\vw}_{m/2})$, $\bfW_0$ being the initialization of $\bfW$, and fix all $a_r$ after initialization.

 To make things cleaner, we will use a smooth activation function $\sigma(\cdot)$ throughout this section\footnote{Similar approximation results also hold for other activation functions like ReLU.}, formally stated as follows.

\begin{asmp}[Smoothness of Activation Function]\label{asmpt:smoothness}
The activation function $\sigma(\cdot)$ is smooth, that is, there exists an absolute constant $C>0$ such that for any $z, z'\in\bbR$
\begin{align*}
|\sigma'(z) - \sigma'(z')| \le C|z - z'|.
\end{align*}
\end{asmp}

Prior to proving the approximation results, we would like to first provide a version of convergence theorem similar to Theorem~\ref{thm:converge_deep}, but for this two-layer setting. It is encouraged that the reader can read Appendix~\ref{sec:two_later_no_proj} for the proof of the following Theorem~\ref{thm:converge_two_layer_projless} first, since it is relatively cleaner than that of the deep setting but the proof logic is analogous.

\begin{thm}[Convergence of Gradient Descent \emph{without Projection} for Optimizing Surrogate Loss for Two-layer Networks]\label{thm:converge_two_layer_projless}
	Suppose the loss function satisfies Assumption~\ref{asmpt:smoothness-loss} and the activation function satisfies Assumption~\ref{asmpt:smoothness}. With high probability, using the two-layer network defined above, for any $\epsilon>0$, if we run gradient descent with step size $\alpha = O\bracket{\epsilon}$, and if $m=\Omega\left(\frac{R^4 }{\epsilon^2}\right)$, we have
	\begin{equation}
	\label{equ:converge_main_two_layer}
	\min_{t=1, \cdots, T}L_\cA(\bfW_t) - L_\ast(\bfW_{\ast}) \leq \epsilon,
	\end{equation}
	where $\bfW_\ast = \min_{\bfW\in B(R)}L_\ast(\bfW)$ and $T= \Theta(\frac{\sqrt{m}}{\alpha})$.
\end{thm}

\begin{remark}
Compared to Theorem~\ref{thm:converge_deep}, we do not need the projection step for this two-layer theorem. We believe using a smooth activation function can also eliminate the need of the projection step in the deep setting from a technical perspective, and from a practical sense we conjecture that the projection step is not needed anyway.
\end{remark}

Now we're ready to proceed to the approximation results, i.e. proving that $L_\ast(\bfW_{\ast})$ is also small, and combined with Equation (\ref{equ:converge_main_two_layer}) we can give an absolute bound on $\min_{t}L_{\cA}(\bfW_t)$. For the reader's convenience, we first introduce the Neural Tangent Kernel (NTK)~\cite{jacot2018neural} w.r.t. our two-layer network.

\begin{defn}[NTK \cite{jacot2018neural}]\label{def:NTK}
The NTK with activation function $\relu{\cdot}$ and initialization distribution $\vw\sim \mathcal{N}(0, \bfI_d)$ is defined as $K_\sigma(\vx, \vy) = \bbE_{\vw\sim \mathcal{N}(0, \bfI_d)}\langle \vx\sigma'(\vw^\top \vx), \vy\sigma'(\vw^\top \vy)\rangle$.
\end{defn}

For a given kernel $K$, there is a reproducing kernel Hilbert space (RKHS) introduced by $K$. We denote it as $\cH(K)$. We refer the readers to \cite{paulsen2016introduction} for an introduction of the theory of RKHS.

%As discussed above, we first consider the infinite width situation, where almost any function can be represented if the induced kernel is universal.
%To find the robust classifier,  Inspired by the recent work on the optimization of infinite wide network, we consider the neural tangent kernel (NTK) \cite{jacot2018neural} introduced by the network. Roughly speaking, the function class obtainable by optimizing infinite width network is a superset of the RKHS induced by the NTK. For a large class of activation functions, the NTK is a universal approximator, and then we use a random feature argument to show that a finite width RKHS can approximate the robust classifier within the RKHS.
%\subsection{Neural Tangent Kernel (NTK) and Reproducing Kernel Hilbert Space (RKHS)}
%We first introduce the neural tangent kernel (NTK)

We formally make the following assumption on the universality of NTK.
\begin{asmp}[Existence of Robust Classifier in NTK]\label{asmpt:robust_exist_ntk}
For any $\epsilon>0$, there exists $f\in\mathcal{H}(\ntk)$, such that $\abs{f(\vx_i') - y_i} \leq \epsilon$, for every $i \in [n]$ and $\vx_i'\in\cB(\vx_i)$.
\end{asmp}
%From now on we will only consider the expressivity on the surface of the unit ball.
% Assumption~\ref{asmpt:robust_exist_ntk} can be verified for a large class of activation functions by showing their induced kernel is universal as done in \cite{micchelli2006universal}. We have shown an example using the quadratic ReLU as activation function, which can be found in Appendix~\ref{sec:quadrelu}.

Also, we make an additional assumption on the activation function $\sigma(\cdot)$:
\begin{asmp}[Lipschitz Property of Activation Function]\label{asmpt:activation_additional}
     The activation function $\sigma(\cdot)$ satisfies $|\sigma'(z)|\le C, \forall z \in \bbR$ for some constant $C$.
\end{asmp}

Under these assumptions, by applying the strategy of approximating the infinite situation by finite sum of random features, we can get the following theorem:
\begin{thm}[Existence of Robust Classifier near Initialization]\label{thm:robust_exist}\label{thm:exist-robust}
Given data set $\mathcal{D} = \{(\vx_i, y_i)\}_{i=1}^{n}$ and a compatible perturbation set function $\mathcal{B}$ with $\vx_i$ and its allowed perturbations taking value on $\mathcal{S}$, for the two-layer network defined in (\ref{def:two_layer_zero_init_output}), if Assumption~\ref{asmpt:smoothness-loss}, \ref{asmpt:smoothness}, \ref{asmpt:robust_exist_ntk}, \ref{asmpt:activation_additional} hold, then for any $\epsilon>0$, there exists $R_{\mathcal{D}, \mathcal{B}, \epsilon}$ such that when the width $m$ satisfies $m = \Omega\bracket{\frac{R_{\mathcal{D}, \mathcal{B}, \epsilon}^4}{\epsilon^2}}$, with probability at least 0.99 over the initialization there exists $\vW$ such that
\begin{equation}
    L_*(\vW) \leq \epsilon\text{ and }\vW\in B(R_{\cD, \cB, \epsilon}).\nonumber
\end{equation}
\end{thm}

Combining Theorem \ref{thm:converge_two_layer_projless} and \ref{thm:robust_exist} we finally know that
\begin{cor}[Adversarial Training Finds a Network of Small Robust Training Loss]\label{cor:exist-robust}
Given data set on the unit sphere equipped with a compatible perturbation set function and an associated perturbation function $\cA$, which also takes value on the unit sphere. Suppose Assumption~\ref{asmpt:smoothness-loss}, \ref{asmpt:smoothness}, \ref{asmpt:robust_exist_ntk}, \ref{asmpt:activation_additional} are satisfied. Then for any $\epsilon > 0$, there exists a $R_{\cD, \cB, \epsilon}$ which only depends on dataset $\cD$, perturbation $\cB$ and $\epsilon$, %corresponding to the RKHS radius,
such that for any $2$-layer fully connected network with width $m =\Omega( \frac{R_{\cD, \cB, \epsilon}^4}{\epsilon^2})$, if we run gradient descent with stepsize $\alpha = O\bracket{\epsilon}$ for $T = \Omega(\frac{R_{\cD, \cB, \epsilon}^2}{\epsilon\alpha})$ steps, then with probability $0.99$,
\begin{align}
    \min_{t = 1, \cdots, T}L_{\cA}(\vW_t) \leq \epsilon.
\end{align}
\end{cor}
\begin{rem}
We point out that Assumption~\ref{asmpt:robust_exist_ntk} is rather general  and can be verified for a large class of activation functions by showing their induced kernel is universal as done in \cite{micchelli2006universal}. Also, here we use an implicit expression of the radius $B_{\cD,\cB,\epsilon}$, but the dependence on $\epsilon$ can be calculated under specific activation function with or without the smoothness assumptions. As an example, using quadratic ReLU as activation function, we solve the explicit dependency on $\epsilon$ in Appendix~\ref{sec:quadrelu} that doesn't rely on Assumption~\ref{asmpt:robust_exist_ntk}.
\end{rem}
Therefore, adversarial training is guaranteed to find a robust classifier under a given attack algorithm when the network width is sufficiently large.

\section{Capacity Requirement of Robustness}
\label{sec:capacity}
In this section, we will show that in order to achieve adversarially robust interpolation (which is formally defined below), one needs more capacity than just normal interpolation. In fact, empirical evidence have already shown that to reliably withstand strong adversarial attacks, networks
require a significantly larger capacity than for correctly classifying benign examples only \cite{madry2017towards}. This implies, in some sense, that using a neural network with larger width is necessary.

Let $\cS_{\delta} = \{(\bfx_1, \cdots, \bfx_n) \in (\mathbb{R}^{d})^{n}: \norm{\bfx_i - \bfx_j}_2>2\delta\}$ and $\mathcal{B}_{\delta}(\bfx) = \{\bfx': \norm{\bfx' - \bfx}_2 \le \delta\}$, where $\delta$ is a constant. We consider datasets in $\cS_{\delta}$ and use $\cB_{\delta}$ as the perturbation set function in this section.

We begin with the definition of the interpolation class and the robust interpolation class.
\begin{defn}[Interpolation class]
%Let the instance space be $\mathcal{X} = \bbR^{d}$, and $S$ be a subset of $\mathcal{X}^{n}$. 
We say that a function class $\mathcal{F}$ of functions $f:\bbR^{d}\rightarrow \{1, -1\}$is an $n$-interpolation class\footnote{Here we let the classification output be $\pm1$, and a usual classifier $f$ outputting a number in $\bbR$ can be treated as $\operatorname{sign}(f)$ here.}, if the following is satisfied:
\begin{align*}
&\forall (\bfx_1, \cdots, \bfx_n) \in \mathcal{S}_\delta, \forall (y_1, \cdots, y_n) \in \{\pm 1\}^n,\\
    &\exists f \in \mathcal{F}, \text{ s.t. } f(\bfx_i) = y_i, \forall i \in [n].
\end{align*}
\end{defn}

\begin{defn}[Robust interpolation class]
%Let the instance space be $\mathcal{X} = \bbR^{d}$, and $S$ be a subset of $\mathcal{X}^{n}$. Let $\mathcal{B}:\mathcal{X}\to\mathcal{P}(\mathcal{X})$ be a perturbation set function. 
We say that a function class $\mathcal{F}$ is an $n$-robust interpolation class, if the following is satisfied: 
\begin{align*}
    &\forall (\bfx_1, \cdots, \bfx_n) \in \mathcal{S_\delta}, \forall (y_1, \cdots, y_n) \in \{\pm 1\}^n,\\
    &\exists f \in \mathcal{F}, \text{s.t.} f(\bfx_i') = y_i, \forall \bfx_i' \in \mathcal{B_\delta}(\bfx_i), \forall i \in [n].
\end{align*}
\end{defn}

We will use the VC-Dimension of a function class $\mathcal{F}$ to measure its complexity. In fact, as shown in \cite{bartlett2019nearly} (Equation(2)), for neural networks there is a tight connection between the number of parameters $W$, the number of layers $H$ and their VC-Dimension
$$\Omega(HW\log(W/H))\leq\text{VC-Dimension}\leq O(HW \log W).$$
In addition, combining with the results in \cite{yun2018finite} (Theorem 3) which shows the existence of a 4-layer neural network with $O(n)$ parameters that can interpolate any $n$ data points, i.e. an $n$-interpolation class, we have that an $n$-interpolation class can be realized by a fixed depth neural network with VC-Dimension upper bound
\begin{align}
    \label{equ:vcd_non_robust}
    \text{VC-Dimension}\leq O(n\log n).
\end{align}
For a general hypothesis class $\cF$, we can evidently see that when $\mathcal{F}$ is an $n$-interpolation class, $\cF$ has VC-Dimension at least $n$. For a neural network that is an $n$-interpolation class, without further architectural constraints, this lower bound of its VC-dimension is tight up to logarithmic factors as indicated in Equation~(\ref{equ:vcd_non_robust}). However, we show that for a robust-interpolation class we will have a much larger VC-Dimension lower bound:
\begin{thm}\label{thm:vcd}
If $\mathcal{F}$ is an $n$-robust interpolation class, then we have the  following lower bound on the VC-Dimension of $\cF$
\begin{align}
    \label{equ:vcd_robust}
    \text{VC-Dimension}\geq \Omega(nd),
\end{align}
where $d$ is the dimension of the input space.
\end{thm}

For neural networks, Equation~\eqref{equ:vcd_robust} shows that any architecture that is an $n$-robust interpolation class should have VC-Dimension at least $\Omega(nd)$.  Compared with Equation~\eqref{equ:vcd_non_robust} which shows an $n$-interpolation class can be realized by a network architecture with VC-Dimension $O(n\log n)$, we can conclude that robust interpolation by neural networks needs more capacity, so increasing the width of neural network is indeed in some sense necessary.

\section{Discussion on Limitations and Future Directions}

This work provides a theoretical analysis of the empirically successful adversarial training algorithm in the training of robust neural networks. Our main results indicate that adversarial training will find a network of low robust surrogate loss, even when the maximization is computed via a heuristic algorithm such as projected gradient ascent. However, there are still some limitations with our current theory, and we also feel our results can lead to several thought-provoking future work, which is discussed as follows.

\emph{Removal of projection.} It is also natural to ask whether the projection step can be removed, as it is empirically unnecessary and also unnecessary for our two-layer analysis. We believe using smooth activations might resolve this issue from a technical perspective, although practically it seems the projection step in the algorithm is unnecessary in any case.

\emph{Generalizing to different attacks.} Firstly, our current guarantee of the surrogate loss is based on the same perturbation function as that used during training. It is natural to ask that whether we can ensure the surrogate loss is low with respect to a larger family of perturbation functions than that used during training.

\emph{Exploiting structures of network and data.} Same as the recent proof of convergence on overparameterized networks in the non-robust setting, our analysis fails to further incorporate useful network structures apart from being sufficiently wide, and as a result increasing depth can only hurt the bound. It would be interesting to provide finer analysis based on additional assumptions on the alignment of the network structure and data distribution.

\emph{Improving the approximation bound.} On the expressivity side, the current argument utilizes that a neural net restricted to a local region can approximate its induced RKHS. Although the RKHS is universal, they do not avoid the curse of dimensionality (see Appendix~\ref{sec:quadrelu}). However, we believe in reality, the required radius of region $R$ to achieve robust approximation is not as large as the theorem demands. So an interesting question is whether the robust expressivity of neural networks can adapt to structures such as low latent dimension of the data mechanism \cite{du2018power,yarotsky2018optimal}, thereby reducing the approximation bound. 

\emph{Capacity requirement of robustness and robust generalization.}
Apart from this paper, there are other works supporting the need for capacity including the perspective of network width \cite{madry2017towards}, depth \cite{xie2019intriguing} and computational complexity \cite{nakkiran2019adversarial}. It is argued in \cite{yin2018rademacher} that robust generalization is also harder using Rademacher complexity. In fact, it appears empirically that robust generalization is even harder than robust training. It is observed that increasing the capacity, though benifiting the dacay of training loss, has much less effect on robust generalization. There are also other factors behind robust generalization, like the number of training data \cite{schmidt2018adversarially}. The questions about robust generalization, as well as to what extent capacity influnces it, are still subject to much debate.

The above are several interesting directions of further improvement to our current result. In fact, many of these questions are largely unanswered even for neural nets in the non-robust setting, so we leave them to future work.

\section{Acknowlegements}
We acknowlegde useful discussions with Siyu Chen, Di He, Runtian Zhai, and Xiyu Zhai. RG and TC are partially supported by the elite undergraduate training program of School of Mathematical Sciences in Peking University. LW acknowledges support by Natioanl Key R\&D Program of China (no. 2018YFB1402600), BJNSF (L172037). JDL acknowledges support of the ARO under MURI Award W911NF-11-1-0303,  the Sloan Research Fellowship, and NSF CCF \#1900145.

\bibliographystyle{plain}
\bibliography{references}  %%% Remove comment to use the external .bib file (using bibtex).

\appendix
\section{Proof of the Convergence Result for Deep Nets in Section~\ref{sec:convergence}}
We first present some useful notations and lemmas for this part. Denote the diagonal matrix $\bfD^{(h)}=\bfD^{(h)}\left(\bfW,\vx\right)$ as $\bfD^{(h)}=\text{diag}\left(\I(\overline{\vx}^{(h)}\ge 0)\right)$ for $h\in[H]$, where $\I$ is the entry-wise indicator function. Sometimes we also denote $\bfD^{(0)}=\bfI$ which is the identity matrix. Therefore, the neural network has the following formula
\begin{align*}
    f\left(\bfW,\vx\right)=\bfa^\top \bfD^{(H)}\bfW^{(H)}\cdots \bfD^{(1)}\bfW^{(1)}\vx^{(0)}
\end{align*}
and the gradient w.r.t. $\bfW^{(h)}$ is 
\begin{align}
    f'^{(h)}\left(\bfW,\vx\right)=\left(\vx^{(h-1)}\bfa^\top \bfD^{(H)}\bfW^{(H)}\cdots \bfD^{(h)}\right)^\top,\quad \quad h\in[H].
     \label{equ:gradient}
\end{align}
First, we will restate some basic results at initialization.

\begin{lem}\label{lem:gaussian_init}
If $m \ge d$, with probability $1 - O(H)e^{-\Omega(m)}$ at initialization, we have $\norm{\bfA}_2 = O(1)$, $\norm{\bfW^{(h)}}_2 = O(1), \forall h\in[H]$, and $\norm{\bfa}_2 = O(\sqrt{m})$.
\end{lem}
\begin{proof}
This is a well-known result of the $l_2$-norm of Gaussian random matrices (Corollary 5.35 in \cite{vershynin2010introduction}), which states that for a matrix $\mathbf{M} \in \bbR^{a\times b}$ with i.i.d. standard Gaussian entries, with probability $1 - 2e^{-t^2/2}$ we have $\norm{\mathbf{M}}_2 \le \sqrt{a} + \sqrt{b} + t$. Combined with the scaling of $\bfA, \bfW^{(h)},$ and $\bfa$, we easily know that each of $\norm{\bfA}_2 = O(1)$, $\norm{\bfW^{(h)}}_2 = O(1)$, and $\norm{\bfa}_2 = O(\sqrt{m})$ holds with probability $1 - e^{-\Omega(m)}$, and we obtain our result by taking the union event.
\end{proof}

\begin{lem}
For any fixed input $\bfx\in\cS$, with probability $1-O(H)e^{-\Omega(m/H)}$ over the randomness of initialization, we have for every $h\in\{0,\ldots,H\}$, $\norm{\overline{\vx}^{(h)}}_2\in[2/3,4/3]$ and $\norm{{\vx}^{(h)}}_2\in[2/3,4/3]$ at initialization.
\label{lem:init}
\end{lem}
\begin{proof}
This is a restatement of Lemma 7.1 in \cite{allen2018convergencetheory} taking the number of data $n=1$.
\end{proof}
 
\begin{lem}\label{lem:2}
If $m=\Omega\left(H\log H\right)$, for any fixed input $\bfx\in\cS$, with probability $1-e^{-\Omega(m/H)}$, at initialization we have for every $h\in[H]$, 
$$\norm{\bfa^\top \bfD^{(H)}\bfW^{(H)}\cdots \bfD^{(h)}\bfW^{(h)}}_2=O(\sqrt{mH}).$$ 
\end{lem}
\begin{proof}
Note that $\norm{\bfa^\top \bfD^{(H)}\bfW^{(H)}\cdots \bfD^{(h)}\bfW^{(h)}}_2\le \norm{\bfa^\top}_2\norm{\bfD^{(H)}}_2\norm{\bfW^{(H)}\cdots \bfD^{(h)}\bfW^{(h)}}_2$ and $\norm{\bfD^{(H)}}_2 \le 1$. This lemma then becomes a direct consequence of Lemma~\ref{lem:gaussian_init} and Lemma 7.3(a) in \cite{allen2018convergencetheory} with number of data $n = 1$.
\end{proof}

Our general idea is that within the local region (where $R = \Omega(1)$)
$$B(R) = \{\bfW:\norm{\bfW^{(h)} - \bfW_0^{(h)}}_F \le \frac{R}{\sqrt{m}}, \forall h \in [H]\}$$
the gradient $f'^{(h)}(\bfW)$ remains stable over $\bfW$ when $\bfx$ is fixed, and the perturbation of $f'^{(h)}(\bfW)$ is small compared to the scale of $f'^{(h)}(\bfW_0)$. This property has been studied in \cite{allen2018convergencetheory} extensively. However, in the non-adversarial setting, they only need to prove this property at finitely many data points $\{\vx_i\}_{i=1}^n$. In our adversarial training setting, though, we also need to prove that it holds for any $\vx'_i\in\cB(\vx_i)$. Specifically, in this section we would like to prove that it holds for any $\vx \in\cS$. Our method is based on viewing the perturbation of $\bfx$ as an equivalent perturbation of the parameter $\bfW^{(1)}$, and then we will be able to make use of the results in \cite{allen2018convergencetheory}. This is elaborated in the following lemma:
\begin{lem}
Given any fixed input $\bfx \in\cS$. If $R=O(\sqrt{m})$, with probability $1-O(H)e^{-\Omega (m/H)}$ over random initialization, for any $\bfx'\in \cS$ satisfying $\norm{\bfx - \bfx'}_2\le\delta$, and any $\bfW\in B(R)$, there exists $\widetilde{\bfW}\in B(R+O(\sqrt{m}\delta))$ such that $\bfW^{(h)} = \widetilde{\bfW}^{(h)}$ for $h = 2, \cdots, H$, and for all $h\in [H]$ we have
\begin{align*}
    \overline{\vx}'^{(h)}(\bfW)=\overline{\vx}^{(h)}(\widetilde{\bfW}),\quad {\vx}'^{(h)}(\bfW)={\vx}^{(h)}(\widetilde{\bfW}),\quad \bfD^{(h)}(\vx',\bfW)=\bfD^{(h)}(\vx,\widetilde{\bfW}).
\end{align*}
In other words, the network with a perturbation from $\bfx$ to $\bfx'$ is same as the network with a perturbation from $\bfW$ to $\widetilde{\bfW}$ since layer $\overline{\bfx}^{(1)}$ and up.
\label{lem:proposition1}
\end{lem}
\begin{proof}
By Lemma~\ref{lem:gaussian_init}, with probability $1-O(H)e^{-\Omega(m)}$, $\norm{\bfA}_2=O(1)$ and $\norm{\bfW_0^{(1)}}_2=O(1)$. Thus $\norm{{\bfW}^{(1)}}_2\le \norm{{\bfW}^{(1)}_0}_2+\frac{R}{\sqrt{m}}=O(1)$. By Lemma~\ref{lem:init}, with probability $1-O(H)e^{-\Omega (m/H)}$, $\norm{\vx^{(0)}}_2\in [2/3,4/3]$. Let
\begin{align*}
    \widetilde{\bfW}^{(1)}={\bfW}^{(1)}+\frac{{\bfW}^{(1)}\left(\vx'^{(0)}-\vx^{(0)}\right)(\vx^{(0)})^\top}{\norm{\vx^{(0)}}_2^2}.
\end{align*}
$\widetilde{\bfW}^{(1)}$ obviously satisfies $\widetilde{\bfW}^{(1)}\vx^{(0)}={\bfW}^{(1)}\vx'^{(0)}$. Then setting $\widetilde{\bfW}^{(2)},\ldots,\widetilde{\bfW}^{(H)}$ equal to ${\bfW}^{(2)},\ldots,{\bfW}^{(H)}$ will make all the following hidden layer vectors and $\bfD^{(h)}$ equal. It is also easy to verify that $$\norm{\widetilde{\bfW}^{(1)}-{\bfW}^{(1)}}_F\le\frac{\norm{\bfW^{(1)}}_2\norm{\bfA}_2\norm{\bfx - \bfx'}_2}{\norm{\bfx^{(0)}}_2} = O(\delta),$$
so we know that $\widetilde{\bfW}\in B(R+O(\sqrt{m}\delta))$.
\end{proof}

By Lemma~\ref{lem:proposition1}, we can directly apply many results in \cite{allen2018convergencetheory} which are only intended for the fixed data originally, to our scenario where the input can be perturbed, as long as we take the parameter radius as $\frac{R}{\sqrt{m}} + O(\delta)$ in their propositions\footnote{Note that in \cite{allen2018convergencetheory} the corresponding region $B(R)$ is defined by the $2$-norm instead of the $F$-norm: $B_2(R) := \{\bfW: \norm{\bfW^{(h)} - \bfW_0^{(h)}}_2\le \frac{R}{\sqrt{m}}, \forall h \in [H]\}$. Since obviously $B_2(R)\subset B(R)$, we can still apply their results to our case directly.}. This can give us the following important lemma:

\begin{lem}[Bound for the Perturbation of Gradient]
Given any fixed input $\bfx \in\cS$. If $m \ge \max(d, \Omega(H\log H))$, $\frac{R}{\sqrt{m}}+\delta\le\frac{c}{H^6(\log m)^{3}}$ for some sufficiently small constant $c$, then with probability at least $1-O(H)e^{-\Omega (m(R/\sqrt{m}+\delta)^{2/3}H)}$ over random initialization, we have for any $\bfW\in B(R)$ and any $\bfx'\in\cS$ with $\norm{\bfx - \bfx'}_2\le\delta$,
\begin{align*}
    \norm{f'^{(h)}\left(\bfW,\vx'\right)-f'^{(h)}\left(\bfW_0,\vx\right)}_F= O\left((\frac{R}{\sqrt{m}}+\delta)^{1/3}H^2\sqrt{m\log m}\right)
\end{align*}
and
\begin{align*}
    \norm{f'^{(h)}\left(\bfW,\vx'\right)}_F= O(\sqrt{mH}).
\end{align*}
\label{lem:3}
\end{lem}
\begin{proof}
By Lemma 8.2(b)(c) of \cite{allen2018convergencetheory}, using the method of Lemma~\ref{lem:proposition1} stated above, when $\frac{R}{\sqrt{m}}+\delta\le\frac{c}{H^{9/2}(\log m)^3}$, with probability $1-e^{-\Omega (m(R/\sqrt{m}+\delta)^{2/3}H)}$, for any $\bfW\in B(R)$ and any $\bfx'\in\cS$ with $\norm{\bfx - \bfx'}_2\le\delta$, we have for $h\in[H]$, \footnote{Here the zero norm $\norm{\cdot}_0$ denotes the number of non-zero entries of a matrix or a vector.}
\begin{align}\label{equ:bound_zero_norm}
    \norm{\bfD^{(h)}(\bfW,\vx')-\bfD^{(h)}(\bfW_0,\vx)}_0=O\left(m(\frac{R}{\sqrt{m}}+\delta)^{2/3}H\right).
\end{align}
and
\begin{align}
    \norm{\vx'^{(h)}(\bfW)-\vx^{(h)}(\bfW_0)}_2= O\left((\frac{R}{\sqrt{m}}+\delta)H^{5/2}\sqrt{\log m}\right),
    \label{equ:star2}
\end{align}
where (\ref{equ:star2}) is also easily verified to hold for $h=0$.
Next, according to Lemma~8.7 of \cite{allen2018convergencetheory}\footnote{We only use the setting when the network output is a scalar.}, when the bound (\ref{equ:bound_zero_norm}) satisfies $O(m(\frac{R}{\sqrt{m}}+\delta)^{2/3}H) \le \frac{m}{H^3\log m}$, with probability $1-e^{-\Omega((R/\sqrt{m}+\delta)^{2/3}Hm\log m)}$, we have for any $\bfW\in B(R)$ and any $\bfx'\in\cS$ with $\norm{\bfx - \bfx'}_2\le\delta$, $\forall h\in[H]$,
\begin{align}
   & \norm{\bfa^\top\bfD^{(H)}(\bfW,\vx')\bfW^{(H)}\cdots \bfD^{(h)}(\bfW,\vx')\bfW^{(h)} - \bfa^\top\bfD^{(H)}(\bfW_0,\vx)\bfW_0^{(H)}\cdots \bfD^{(h)}(\bfW_0,\vx)\bfW_0^{(h)}}_2\nonumber\\&=O\left((\frac{R}{\sqrt{m}}+\delta)^{1/3}H^2\sqrt{m\log m}\right)
   \label{equ:star1}
\end{align}
Note that with our condition $\frac{R}{\sqrt{m}}+\delta\le\frac{c}{H^6(\log m)^{3}}$, the previous requirements are all satisfied. Also, 
 combining (\ref{equ:star2}) with Lemma~\ref{lem:init}, we know for $h = 0, \cdots, H$,
\begin{align}
     \norm{\vx'^{(h)}(\bfW)}_2 \le O(1) + O\left((\frac{R}{\sqrt{m}}+\delta)H^{5/2}\sqrt{\log m}\right) = O(1)
     \label{equ:star3}
\end{align}
Combining Equation~(\ref{equ:star2}),~(\ref{equ:star1}),~(\ref{equ:star3}), and Lemma~\ref{lem:2}, we obtain
\begin{align*}
    &\norm{f'^{(h)}\left(\bfW,\vx'\right)-f'^{(h)}\left(\bfW_0,\vx\right)}_F\\
    \le& \norm{\bfa^\top\bfD^{(H)}(\bfW,\vx')\bfW^{(H)}\cdots \bfD^{(h)}(\bfW,\vx')\bfW^{(h)} - \bfa^\top\bfD^{(H)}(\bfW_0,\vx)\bfW_0^{(H)}\cdots \bfD^{(h)}(\bfW_0,\vx)\bfW_0^{(h)}}_2\\&\cdot\norm{\vx'^{(h-1)}(\bfW)}_2+\norm{ \bfa^\top\bfD^{(H)}(\bfW_0,\vx)\bfW_0^{(H)}\cdots \bfD^{(h)}(\bfW_0,\vx)\bfW_0^{(h)}}_2\norm{\vx'^{(h-1)}(\bfW)-\vx^{(h-1)}(\bfW_0)}_2\\
    =& O\left((\frac{R}{\sqrt{m}}+\delta)^{1/3}H^2\sqrt{m\log m}\right) + O(\sqrt{mH})\cdot O\left((\frac{R}{\sqrt{m}}+\delta)H^{5/2}\sqrt{\log m}\right)\\
    =& O\left((\frac{R}{\sqrt{m}}+\delta)^{1/3}H^2\sqrt{m\log m}\right).
\end{align*}
In addition, also by (\ref{equ:star3}) and Lemma~\ref{lem:2},
\begin{align*}
    \norm{f'^{(h)}(\bfW_0,\vx)}_F&\le \norm{ \bfa^\top\bfD^{(H)}(\bfW_0,\vx)\bfW_0^{(H)}\cdots \bfD^{(h)}(\bfW_0,\vx)\bfW_0^{(h)}}_2\norm{\vx^{(h-1)}(\bfW_0)}_2\\
    &=O(\sqrt{mH}).
\end{align*}
Therefore,
\begin{align*}
    \norm{f'^{(h)}(\bfW,\vx')}_F= O(\sqrt{mH})+O\left((\frac{R}{\sqrt{m}}+\delta)^{1/3}H^2\sqrt{m\log m}\right)=O(\sqrt{mH}).
\end{align*}
\end{proof}
With Lemma~\ref{lem:3}, we are ready to state an important bound that implies the loss function is close to being convex within the neighborhood $B(R)$ for any $\bfx \in \cS$. We use the $\epsilon$-net to turn the result from a fixed $\bfx$ to all $\bfx\in\cS$, 

\begin{lem}\label{lem:5}
If $m = \Omega\bracket{\frac{d^{3/2}\log^{3/2}(\sqrt{m}/R)}{RH^{3/2}}}$ and $R = O\bracket{\frac{\sqrt{m}}{H^6(\log m)^{3}}}$, then with probability at least $1-O(H)e^{-\Omega((mR)^{2/3}H)}$ over random initialization, we have for any $\bfW_1,\bfW_2\in B(R)$, any $\bfx\in\cS$ and any $y\in \bbR$,
\begin{align*}
    l\left(f(\bfW_2,\vx),y\right)&\ge l\left(f(\bfW_1,\vx),y\right)+\langle \nabla_\bfW l(f(\bfW_1,\vx),y),\bfW_2-\bfW_1\rangle\\&-\norm{\bfW_2-\bfW_1}_F O((mR)^{1/3}H^{5/2}\sqrt{\log m} ).
\end{align*}
\end{lem}
\begin{proof}
A $\delta$-net of $\cS$ is a group of points $\{\bfx_i\}_{i=1}^{N}$ in $\cS$ that satisfies: for any $\bfx\in\cS$, there exists some $\bfx_i$ that satisfies $\norm{\bfx - \bfx_i}_2\le\delta$. From classic covering number results we know that we can construct such a $\delta$-net with the number of total points $N = (O(1/\delta))^{d}$, where $d$ is the input dimension. As long as $\frac{R}{\sqrt{m}}+\delta\le \frac{c}{H^6(\log m)^{3}}$ for some sufficiently small constant $c$, we can derive that for any $i\in[N]$, with probability $1-O(H)e^{-\Omega (m(R/\sqrt{m}+\delta)^{2/3}H)}$, for any $\bfW_1,\bfW_2\in B(R)$, any $\bfx_i'\in\cS$ with $\norm{\bfx_i - \bfx_i'}_2 \le \delta$, and any $y \in \bbR$,
\begin{align*}
   & l\left(f(\bfW_2,\vx'_i),y\right)- l\left(f(\bfW_1,\vx'_i),y\right)-\langle \nabla_\bfW l(f(\bfW_1,\vx'_i),y),\bfW_2-\bfW_1\rangle\\
    \ge &\frac{\partial }{\partial f}l\left(f(\bfW_1,\vx'_i),y\right)\left[f(\bfW_2,\vx'_i)-f(\bfW_1,\vx'_i)-\langle \nabla_\bfW f(\bfW_1,\vx'_i),\bfW_2-\bfW_1\rangle\right]\\
    =&\frac{\partial }{\partial f}l\left(f(\bfW_1,\vx'_i),y\right)\langle\int_{0}^1 \left(\nabla_\bfW f(t\bfW_2+(1-t)\bfW_1,\vx'_i)-\nabla_\bfW f(\bfW_1,\vx'_i)\right)dt ,\bfW_2-\bfW_1\rangle\\
    \ge& -\norm{\bfW_2-\bfW_1}_F O((\frac{R}{\sqrt{m}}+\delta)^{1/3}H^{5/2}\sqrt{m\log m} ),
\end{align*}
where the first inequality uses the convexity of $l$ w.r.t $f$ and the last inequality is due to the boundedness of $|\partial l/\partial f|$ is bounded, Lemma~\ref{lem:3}, and $\nabla_\bfW f=(f'^{(1)},\ldots,f'^{(H)})$. We take $\delta = \frac{R}{\sqrt{m}}$. With $R = O\bracket{\frac{\sqrt{m}}{H^6(\log m)^{3}}}$ the requirement $\frac{R}{\sqrt{m}}+\delta\le \frac{c}{H^6(\log m)^{3}}$ can be satisfied. Therefore, taking union event over all $N$ points, our proposition holds with probability
\begin{align*}
    &1-O(H)(O(\sqrt{m}/R))^{d}e^{-\Omega ((mR)^{2/3}H)} \\=& 1-O(H) e^{-\Omega ((mR)^{2/3}H) + d\log(O(\sqrt{m}/R))} \\ =& 1-O(H) e^{-\Omega ((mR)^{2/3}H)},
\end{align*}
where the last equation is due to the condition $m = \Omega\bracket{\frac{d^{3/2}\log^{3/2}(\sqrt{m}/R)}{RH^{3/2}}}$.
\end{proof}

With the above preparations, we are ready to prove the main theorem.
\proof[Proof of Theorem~\ref{thm:converge_deep}]
We denote $\bfW_t$ as the parameter after $t$ steps of projected gradient descent, starting from the initialization $\bfW_0$. We perform a total of $T$ steps with step size $\alpha$. 

For projected gradient descent, ${\vW}_t \in B(R)$ holds for all $t = 0, 1, \cdots, T$. Recall that the update rule is $\vW_{t+1} =\cP_{B(R)}\left(\vV_{t+1}\right)$ for $\vV_{t+1} = \vW_t - \alpha\nabla_\vW L_\cA (\vW_t)$. Let $d_t := \|{\vW}_t-{\vW}_\ast\|_F$. We have
\begin{align}d_{t+1}^2  =&
\|{\vW}_{t+1}-{\vW}_\ast\|_F^2\nonumber\\
\le& \|{\vV}_{t+1}-{\vW}_\ast\|_F^2\nonumber\\
 =& \|{\vW}_t-{\vW}_\ast\|_F^2 + 2\langle{\vV}_{t+1}-{\vW}_t,{\vW}_t-{\vW}_\ast\rangle+\|{\vV}_{t+1}-{\vW}_t\|_F^2\nonumber\\
 =& d_t^2 + 2\alpha\langle\nabla_{\vW}L_\cA({\vW}_t),({\vW}_\ast - {\vW}_t)\rangle + \alpha^2\|\nabla_{\vW}L_\cA({\vW}_t)\|_F^2\nonumber\\
=&d_t^2 + \frac{2\alpha}{n}\sum_{i=1}^n\langle\nabla_{\vW}l({\vW}_t, \cA({\vW}_t, \vx_i)),({\vW}_\ast - {\vW}_t)\rangle + \alpha^2\|\frac{1}{n}\sum_{i=1}^n\frac{\partial l}{\partial f} \nabla_{\vW}f({\vW}_t, \cA({\vW}_t, \vx_i))\|_F^2\nonumber\\
\le& d_t^2 + \frac{2\alpha}{n}\sum_{i=1}^n[l({\vW}_\ast, \cA({\vW}_t, \vx_i))-l({\vW}_t, \cA({\vW}_t, \vx_i)) \nonumber\\&+\norm{\bfW_\ast-\bfW_t}_F O((mR)^{1/3}H^{5/2}\sqrt{\log m})]+\alpha^2 O(mH^2)\nonumber\\
\le & d_t^2 + \frac{2\alpha}{n}\sum_{i=1}^n\left(l({\vW}_\ast, \cA_\ast({\vW}_\ast, \vx_i))-l({\vW}_t, \cA({\vW}_t, \vx_i)) \right)\nonumber\\&+ O(\alpha m^{-1/6}R^{4/3}H^{5/2}\sqrt{\log m}+\alpha^2mH^2)\nonumber\\
=&d_t^2 +2\alpha \left(L_\ast(\bfW_\ast)-L_\cA(\bfW_t) \right)+ O(\alpha m^{-1/6}R^{4/3}H^{5/2}\sqrt{\log m}+\alpha^2mH^2)\nonumber
\end{align}
where the second inequality is due to Lemma~\ref{lem:5} and Lemma~\ref{lem:3}, the third inequality is due to the definition of $\cA_\ast$. Note that in order to satisfy the condition for Lemma~\ref{lem:5} and Lemma~\ref{lem:3}, our choice $m = \max\{\Omega\bracket{\frac{H^{16}R^{9}}{\epsilon^{7}}}, \Omega(d^2)\}$ suffices. By induction on the above inequality, we have
\begin{align*}
    d_T^2\le d_0^2+2\alpha \sum_{t=0}^{T-1}\left(L_\ast(\bfW_\ast)-L_\cA(\bfW_t)\right)+O(T(\alpha m^{-1/6}R^{4/3}H^{5/2}\sqrt{\log m}+\alpha^2mH^2)),
\end{align*}
which implies that
\begin{align*}
    \min_{0\le t\le T}\left(L_\ast(\bfW_\ast)-L_\cA(\bfW_t)\right)&\le \frac{d_0^2-d_T^2}{\alpha T}+O(m^{-1/6}R^{4/3}H^{5/2}\sqrt{\log m}+\alpha mH^2)\\
   & \le \frac{R^2}{m\alpha T}+O(m^{-1/6}R^{4/3}H^{5/2}\sqrt{\log m}+\alpha mH^2)\\
   &\le \epsilon,
\end{align*}
where in the last inequality we use our choice of $\alpha = O\bracket{\frac{\epsilon}{mH^2}}$, $T = \Theta\bracket{\frac{R^2}{m\alpha\epsilon}}$, and also $m^{-1/6}R^{4/3}H^{5/2}\sqrt{\log m} \le O(\epsilon)$, which is satisfied by $m = \Omega\bracket{\frac{H^{16}R^{9}}{\epsilon^{7}}}$.

\label{sec:deepopt}
\section{Proof of Theorem~\ref{thm:converge_two_layer_projless}: Convergence Result for Two-Layer Networks}\label{sec:two_later_no_proj}
\begin{proof}
We denote $\bfW_t$ as the parameter after $t$ steps of projected gradient descent, starting from the initialization $\bfW_0$. We perform a total of $T$ steps with step size $\alpha$, where each step is an update $\bfW_{t+1} = \bfW_t - \alpha\nabla_\bfW L_\cA(\bfW_t)$. Firstly, the formula for the network gradient is
\begin{align*}
    \nabla_\vW f(\bfW, \bfx) = \frac{1}{\sqrt{m}} \operatorname{diag}(\bfa)\sigma'(\bfW\bfx)\bfx^\top,
\end{align*}
where $\bfa = (a_1, \cdots, a_{\frac{m}{2}}, a_1', \cdots, a_{\frac{m}{2}}')^\top$ is the parameter for the output layer.
We can compute the Lipschitz property of $\nabla_\vW f$ w.r.t $\bfW$: For any fixed $\bfx \in \mathcal{S}$,
\begin{align*}
 \|\nabla_{\vW} f(\vW) - \nabla_{\vW} f(\vW')\|_F &\le\frac{1}{\sqrt{m}} \norm{\operatorname{diag}(\bfa)}_2\| \sigma'(\vW \vx)) -\sigma'(\vW'\vx) \|_2\|\vx\|_2\\&
 \le \frac{1}{\sqrt{m}}\cdot 1\cdot C\norm{\bfW\bfx - \bfW' \bfx}_2\cdot 1\\&
 \le O\left(\frac{1}{\sqrt{m}}\right)\|\vW-\vW'\|_F,
 \end{align*}
% which implies that $f(\vW)$ is $O\left(\frac{1}{\sqrt{m}}\right)$-weakly convex, i.e. given any $x$ 
% \begin{align*}
%     f(\vW) + O\left(\frac{1}{\sqrt{m}}\right)\|\vW-\vW'\|_F^2
% \end{align*}
% is convex.  
For a fixed data point $(\bfx, y)$, we denote $\ell(f(\bfW))$ as short for $\ell(f(\bfW, \bfx), y)$. Since $\ell$ is convex and has bounded derivative in $f$, we have
\begin{align}
&\ell(f(\bfW'))-\ell(f(\bfW))\nonumber\\
\ge& \ell'(f(\vW))(f(\vW') - f(\vW)) \nonumber\\
=& \ell'(f(\vW))(\langle\nabla_\vW f(\vW),\vW' - \vW\rangle +\langle\int_0^1 (\nabla_{\bfW}f(s\bfW + (1-s)\bfW') - \nabla_{\bfW}f(\bfW))ds, \vW' - \vW\rangle)\nonumber\\
\ge& \langle\nabla_{\bfW} \ell(f(\bfW)), \bfW - \bfW'\rangle - O\left(\frac{1}{\sqrt{m}}\right)\norm{\vW - \vW'}_F^2\label{equ:convexity_l_two_layer}. 
\end{align}
In addition, we can also easily know that for $\bfW \in B(R)$ and $R = O(\sqrt{M})$, $\bfx \in \mathcal{S}$, we have
\begin{align}\label{equ:bound_on_l_two_layer}
    \norm{\nabla_{\bfW}\ell(f(\bfW))}_2 \le |\ell'|\frac{1}{\sqrt{m}}\norm{\operatorname{diag}(\bfa)}_2(\sqrt{m}|\sigma'(0)| + C\norm{\bfW\bfx}_2)\norm{\bfx}_2 = O(1)
\end{align}
since the initialization satisfies $\norm{\bfW_0}_2 = O(\sqrt{m})$ with high probability given $m = \Omega(d)$ (see Lemma~\ref{lem:gaussian_init}), thereby $\norm{\bfW}_2 = O(\sqrt{m})$.  
% Old thing using "hessian"
%\begin{align*}
%&\nabla_\bfW^2l(f(\bfW, x'), y(x)) \\=& \ell' \nabla^2 _\bfW f(\bfW, x') + \nabla_\bfW f(\bfW, x')\otimes\nabla_\bfW f(\bfW, x') \\
%\succeq& -O\left(\frac{1}{m}\right)\mathbf{I}.\end{align*}

Denote $d_t = \|\bfW_t-\bfW_*\|_F$. Without a projection step, there could be two possible scenarios during the optimization process: Either $\bfW_t \in B(3R)$ holds for all $t = 1, \cdots, T$, or there exists some $T_0 < T$ such that $\bfW_t \in B(3R)$ for $t \le T_0$ but $\bfW_{T_0+1} \notin B(3R)$. Either way, while $\bfW_t$ is still in $B(3R)$ up to $t-1$, we have
\begin{align}\nonumber
d_{t}^2 & =
\|\bfW_{t}-\bfW_\ast\|_F^2\\\nonumber
& = \|\bfW_{t-1}-\bfW_\ast\|_F^2 + 2(\bfW_{t}-\bfW_{t-1})\cdot(\bfW_{t-1}-\bfW_\ast)+\|\bfW_{t}-\bfW_{t-1}\|_F^2\\\nonumber
& = d_{t-1}^2 + 2\alpha\nabla_\bfW L_\cA(\bfW_{t-1})\cdot(\bfW_\ast - \bfW_{t-1}) + \alpha^2\|\nabla_\bfW L_\cA(\bfW_{t-1})\|_F^2 \\\nonumber
& \leq d_{t-1}^2 + \frac{2\alpha}{n}\sum_{i=1}^{n}[\ell(f(\bfW_\ast, \mathcal{A}(\bfW_{t-1}, \bfx_i)), y_i)-\ell(f(\bfW_{t-1}, \cA(\bfW_{t-1}, \bfx_i)), y_i)\\\nonumber &\quad+  O\left(\frac{1}{\sqrt{m}}\right)\|\bfW_\ast - \bfW_{t-1}\|_F^2] + O(\alpha^2)\\\label{induction_ineq_two_layer}
& \leq \left(1+\frac{c\alpha}{\sqrt{m}}\right)d_{t-1}^2 + 2\alpha(L_\ast(\bfW_\ast)-L_\cA(\bfW_{t-1})) + O(\alpha^2),
\end{align}
where the first inequality is based on (\ref{equ:convexity_l_two_layer}) and (\ref{equ:bound_on_l_two_layer}), the second inequality is based on the definition of $L_\ast(\bfW_\ast)$, and $c$ is some constant.
Let $S_t = (1+\frac{c\alpha}{\sqrt{m}})^{t}$ which is a geometric series, and dividing (\ref{induction_ineq_two_layer}) by $S_t$ we have
\begin{align*}
\frac{d_{t}^2}{S_t} &\le \frac{d_{t-1}^2}{S_{t-1}}-2\alpha\frac{L_\cA(\bfW_{t-1})-L_\ast(\bfW_\ast)}{S_t} + \frac{O(\alpha^2)}{S_t},
\end{align*}
which, by induction, gives us
\begin{align*}
\frac{d_{t}^2}{S_t}
&\le d_0^2 - 2\alpha\sum_{i=0}^{t-1}\frac{L_\cA(\bfW_i)-L_\ast(\bfW_\ast)}{S_{i+1}} + O(\alpha^2)\sum_{i=0}^{t-1}\frac{1}{S_{i+1}}\\
&\le d_0^2 - 2\alpha\min_{i=0,\cdots,t-1}(L_\cA(\bfW_i)-L_\ast(\bfW_\ast))\sum_{i=0}^{t-1}\frac{1}{S_{i+1}} + O(\alpha^2)\sum_{i=0}^{t-1}\frac{1}{S_{i+1}},
\end{align*}
and note that $\sum_{i=0}^{t-1}\frac{1}{S_{i+1}} = \frac{\sqrt{m}}{c\alpha}\bracket{1-\frac{1}{S_{t}}}$, which yields
\begin{align}\label{equ:important_bound}
    \min_{i=0,\cdots,t-1}L_\cA(\bfW_{i})-L_\ast(\bfW_\ast) \le O(\alpha) + \frac{c\bracket{d_0^2 - \frac{d_{t}^2}{S_t}}}{\sqrt{m}(1 - \frac{1}{S_{t}})}.
\end{align}

Now we will consider the two cases separately:

\emph{Case 1.} $\bfW_t \in B(3R)$ holds for all $t = 1, \cdots, T$.
We have chosen $T = \frac{\sqrt{m}}{c\alpha}$, and then $S_{T} \approx e$. Also, since $d_0^2 - \frac{d_{T}^2}{S_T} \le d_0^2 = O(R^2)$, by choosing $m = \Omega(R^4/\epsilon^2)$ and $\alpha = O\bracket{\epsilon}$, and taking $t = T$ in (\ref{equ:important_bound}), we can obtain 
$\min_{t=0,\cdots,T}L_\cA(\bfW_{t})-L_\ast(\bfW_\ast) \le \epsilon$.

\emph{Case 2.} There exists some $T_0 < T$ such that $\bfW_t \in B(3R)$ for $t \le T_0$ but $\bfW_{T_0+1} \notin B(3R)$. Since $\bfW_* \in B(R)$, we know that $d_0 \le R$ and $d_{T_0 + 1} \ge 2R$. Still using the choice of parameters above, we have $d_0^2 - \frac{d_{T_0+1}^2}{S_{T_0+1}} \leq R^2 - (4/e)R^2 \leq 0$. Hence, taking $t = T_0+1$ in (\ref{equ:important_bound}), we obtain $\min_{t=0,\cdots,T}L_\cA(\bfW_{t})-L_\ast(\bfW_\ast) \leq \min_{t=0,\cdots,T_0}L_\cA(\bfW_{t})-L_\ast(\bfW_\ast) \leq \epsilon$.

So in any case the result is correct, thus we have proved the convergence without the need of projection.
\end{proof}
\section{Proof of Gradient Descent Finding Robust Classifier in Section~\ref{sec:adv_robust}}\label{sec:adv_robust_appendix}
\subsection{Proof of Theorem~\ref{thm:robust_exist}}
As discussed in Section~\ref{sec:adv_robust}, we will use the idea of random feature~\cite{rahimi2008uniform} to approximate ${g}\in\cH(\ntk)$ on the unit sphere. We consider functions of the form
$$h(\vx) = \int_{\bbR^d}c(\vw)^\top \vx\sigma'(\vw^\top \vx) d\vw,$$
where $c(\vw): \bbR^d\to\bbR^d$ is any function from $\bbR^d$ to $\bbR^d$. We define the RF-norm of $h$ as $\rfnorm{h} = \sup_{\vw}\frac{\norm{c(\vw)}_2}{p_0(\vw)}$ where $p_0(\vw)$ is the probability density function of $\mathcal{N}(0, \mathbf{I}_d)$, which is the distribution of initialization. Define the function class with finite $\mathcal{N}(0, \mathbf{I}_d)$-norm as
$    \rf = \left\{ h(\vx) = \int_{\bbR^d}c(\vw)^\top \vx\sigma'(\vw^\top \vx) d\vw : \rfnorm{h} < \infty\right\}.$
We firstly show that $\rf$ is dense in $\cH(\ntk)$.
\begin{lem}[Universality of $\rf$]\label{lem:rf_universal}
Let $\rf$ and $\cH(\ntk)$ be defined as above. Then $\rf$ is dense in $\cH(\ntk)$, and further, dense in $\cH(\ntk)$ w.r.t. $\norm{\cdot}_{\infty,\cS}$, where $\norm{f}_{\infty, \cS} = \sup_{\vx\in \cS}\abs{f(\vx)}$.
\end{lem}
\begin{proof}
Observe that by the definition of the RKHS introduced by $\ntk$, functions with form $h(\vx) = \sum_t a_tK(\vx, \vx_t),\quad \vx_t\in\cS$ are dense in $\cH(\ntk)$. But these functions can also be written in the form $h(\vx) = \int_{\bbR^d}c(\vw)^\top \vx\sigma'(\vw^\top \vx) d\vw$ where $c(\vw) = p_0(\vw) \sum_t a_t \vx_t\sigma'(\vw^\top \vx_t)$. Note that $\norm{c(\vw)}_2\leq p(\vw)\sum_t \norm{a_t \vx_t\sigma'(\vw^\top \vx_t)}_2 < \infty$ since $\cS$ is a compact set and $\sigma'$ is bounded, this verifies that $h$ is an element in $\rf$. So $\rf$ contains a dense set of $\cH(\ntk)$ and therefore dense in $\cH(\ntk)$. Then note that the evaluation operator $K_{\sigma, \vx}$ is uniformly bounded for $\vx\in\cS$, and $h(\bfx) = \langle K_{\sigma, \vx}, h\rangle_{\cH}$, so the RKHS norm can be used to control the norm $\norm{\cdot}_{\infty,\cS}$ and is therefore stronger, thus the proof is complete.
\end{proof}
%By the universality of the mixture class $\rf$, we can conclude that for any constant $\bar{\epsilon}$ there exists a function $g_2\in\rf$ such that $\abs{g_1(\vx)-g_2(\vx)}<\bar{\epsilon}$ for any $x\in\cS$. Then, we can finally come into the finite situation.

We then show that we can approximate elements of $\rf$ by finite random features. Our results are inspired by~\cite{rahimi2008uniform}. For the next theorem, recall Assumption~\ref{asmpt:smoothness}, \ref{asmpt:activation_additional}, the constant $C$ satisfies $\sigma'$ is $C$-Lipschitz, $\abs{\sigma'(\cdot)}\leq C$.

\begin{prop}[Approximation by Finite Sum]\label{thm:random_feature}
Let $h(\vx) = \int_{\bbR^d}c(\vw)^\top \vx\sigma'(\vw^\top \vx) d\vw \in\rf$. Then for any $\delta>0$, with probability at least $1-\delta$ over $\vw_1,\cdots,\vw_M$ drawn i.i.d. from $\mathcal{N}(0, \mathbf{I}_d)$, there exists $c_1,\cdots,c_M$ where $c_i\in\bbR^d$ and $\norm{c_i}_2 \leq \frac{\rfnorm{h}}{M}$, so that the function
$\hat{h} = \sum_{i=1}^M c_i^\top \vx\sigma'(\vw_i^\top \vx),$
satisfies
$$\norm{\hat{h} - h}_{\infty,\cS} \leq \frac{C\rfnorm{h}}{\sqrt{M}}\left(2\sqrt{d} + \sqrt{2\log(1/\delta)}\right).$$
\end{prop}
\begin{proof}
This result is obtained by importance sampling, where we construct $\hat{h}$ with $c_i = \frac{c(\vw_i)}{Mp_0(\vw_i)}$. We first notice that $\norm{c_i}_2 = \frac{\norm{c(\vw_i)}_2}{Mp_0(\vw_i)} \leq \frac{\rfnorm{h}}{M}$ which satisfies the condition of the theorem. We then define the random variable
$$v(\vw_1, \cdots, \vw_M) = \norm{\hat{h} - h}_{\infty,\cS}.$$
We bound this deviation from its expectation using McDiarmid's inequality.

To do so, we should first show that $v$ is robust to the perturbation of one of its arguments. In fact, for $\vw_1, \cdots, \vw_M$ and $\Tilde{\vw_i}$ we have
\begin{align*}
    &\abs{v(\vw_1,\cdots,\vw_M) - v(\vw_1,\cdots,\Tilde{\vw_i},\cdots,\vw_M)}\\
    \le&\frac{1}{M}\max_{\bfx\in\cS}\abs{\frac{c(\vw_i)^\top \vx\sigma'(\vw_i^\top \vx)}{p_0(\vw_i)} - \frac{c(\Tilde{\vw_i})^\top \vx\sigma'(\Tilde{\vw_i}^\top \vx)}{p_0(\Tilde{\vw_i})}}\\
    \leq&\frac{1}{M}\rfnorm{h}\max_{\bfx\in\cS}\left(\abs{\sigma'(\vw_i^\top \vx)} + \abs{\sigma'(\Tilde{\vw_i}^\top \vx)}\right)
    \\
    \leq&\frac{2C\rfnorm{h}}{M} =: \xi
\end{align*}
by using triangle, Cauchy-Schwartz inequality, $\abs{\sigma'(\cdot)}\le C$ and $\norm{\vx}_2 = 1$.

Next, we bound the expectation of $v$. First, observe that the choice of $c_1, \cdots, c_M$ ensures that $\bbE_{\bfw_1, \cdots, \bfw_M}\hat{h} = h$. By symmetrization~\cite{mohri2012new}, we have
\begin{align}\nonumber
    \bbE v =& \bbE \sup_{\vx\in\cS}\abs{\hat{h}(\vx) - \bbE\hat{h}(\vx)}\\\label{equ:rademacher_middle}
    \leq& 2\bbE_{\vw, \epsilon}\sup_{\vx\in\cS}\abs{\sum_{i=1}^M\epsilon_i c_i^\top \vx\sigma'(\vw_i^\top \vx)},
\end{align}
where $\epsilon_1,\cdots,\epsilon_M$ is a sequence of Rademacher random variables.

Since $\abs{c_i^\top \vx}\leq\norm{c_i}_2\leq\frac{\rfnorm{h}}{M}$ and $\sigma'$ is $C$-Lipschitz, we have that $c_i^\top \vx\sigma'(\cdot)$ is $\frac{C\rfnorm{h}}{M}$-Lipschitz in the scalar argument and zero when the scalar argument is zero. Following (\ref{equ:rademacher_middle}), by Talagrand’s lemma (Lemma 5.7) in \cite{mohri2018foundations} together with Cauchy-Schwartz, Jensen's inequality, we have
\begin{align*}
    \bbE v\leq&2\bbE_{\vw, \epsilon}\sup_{\vx\in\cS}\abs{\sum_{i=1}^M\epsilon_i c_i^\top \vx\sigma'(\vw_i^\top \vx)}\\
    \leq&\frac{2C\rfnorm{h}}{M}\bbE\sup_{\vx\in\cS}\abs{\sum_{i=1}^M\epsilon_i \vw_i^\top \vx}\\
    \leq&\frac{2C\rfnorm{h}}{M}\bbE\norm{\sum_{i=1}^M\epsilon_i \vw_i}_2\\
    \leq&\frac{2C\rfnorm{h}}{\sqrt{M}}\sqrt{\bbE_{\vw\sim\mathcal{N}(0, \mathbf{I}_d)}\norm{\vw}_2^2} =: \mu
\end{align*}
Then McDiarmid's inequailty implies
$$\bbP[v \geq\mu+\epsilon]\leq\bbP[v\geq\bbE v+\epsilon]\leq\exp(-\frac{2\epsilon^2}{M\xi}).$$
The proposition is proved by solving the $\epsilon$ while setting the right hand to the given $\delta$.
\end{proof}

\paragraph{Proof of Theorem~\ref{thm:robust_exist}}. Finally, we construct $\vW_\ast$ within a ball of the initialization $\vW_0$ that suffers little robust loss $L_\ast(\vW_\ast)$.
%We note that $\vW$ here stand for the parameters we will optimize, i.e. $(\vw_1, \dots, \vw_m, \bar{\vw}_1, \cdots, \bar{\vw}_m)$, we further denote the initialized parameter as $\vW_0$. Then
Using the symmetric initialization in (\ref{def:two_layer_zero_init_output}), we have $f(\vW_0, \vx)= 0$ for all $\vx$. We then use the neural Taylor expansion w.r.t. the parameters:
\begin{align}
    %\label{equ:neural_taylor}
    %f(\vW, \vx) =&  \underbrace{(\vW - \vW_0)^\top \nabla_\vW f(\vW_0, \vx)}_{(i)} + \underbrace{(\vW - \vW_0)^\top \nabla^2_\vW f(\vW_1, \vx)(\vW - \vW_0)}_{(ii)}
    %f(\vW, \vx) &\approx f(\vW_0, \vx) + (\vW - \vW_0)^\top \nabla_\vW f(\vW_0, \vx) + (\vW - \vW_0)^\top \nabla^2_\vW f(\vW_1, \vx)(\vW - \vW_0)\nonumber\\
    f(\vW, \vx) - f(\vW_0, \vx)&\approx  \underbrace{\frac{1}{\sqrt{m}}\left( \sum_{i=1}^{m/2} a_i (\vw_i - \vw_{i0})^\top \vx \sigma'(\vw_{i0}^\top \vx) + \sum_{i=1}^{m/2} a'_i (\bar{\vw}_i - \bar{\vw}_{i0})^\top \vx \sigma'(\bar{\vw}_{i0}^\top \vx)\right)}_{(i)},\nonumber
\end{align}
where $\vw_{i0}$ denotes the value of $\vw_i$ at initialization. We omitted the second order term. The term (i) has the form of the random feature approximation, and so Proposition \ref{thm:random_feature} can be used to construct a robust interpolant.

In summary, we give the entire proof of Theorem~\ref{thm:robust_exist} as follows.
\begin{proof}
%\tnote{Need to change the loss part when optimization result has done.}
Let $L$ be the Lipschitz coefficient of the loss function $\ell$. Let $\bar{\epsilon} = \frac{1}{3L}$.

By Assumption~\ref{asmpt:robust_exist_ntk} with $\bar{\epsilon}$, there exists $g_1\in\cH(\ntk)$ such that
\begin{align*}
    \abs{g_1(\vx_i') - y_i} \leq \bar{\epsilon},
\end{align*}
for every $\vx_i'\in\cB(\vx_i)$, $i\in[n]$, where $\cB(\vx_i)$ is the perturbation set.

By Lemma~\ref{lem:rf_universal}, for $\bar{\epsilon}$ there exists $g_2\in\rf$ such that $\norm{g_1 - g_2}_{\infty, \cS}\leq\bar{\epsilon}$. Then, by Theorem~\ref{thm:random_feature}, we have $c_1, \cdots, c_{m/2}$ where $c_i\in\bbR^d$ and
$$\norm{c_i}_2\leq\frac{\rfnorm{g_2}}{m},$$
such that $g_3 = \sum_{i=1}^{m/2} c_i^\top \vx\sigma'(\vw_i^\top \vx)$ satisfies
$$\norm{g_2 - g_3}_{\infty,\cS} \leq \frac{C\rfnorm{g_2}}{\sqrt{m/2}}\left(2\sqrt{d} + \sqrt{2\log{1/\delta}}\right),$$
with probability at least $1-\delta$ on the initialization $\vw_i$'s.

We decompose $f$ into the linear part and its residual:
\begin{align*}
    f(\vW, \vx) =& \frac{1}{\sqrt{m}}\left( \sum_{i=1}^{m/2} a_i (\vw_i - \vw_{i0})^\top \vx \sigma'(\vw_{i0}^\top \vx) + \sum_{i=1}^{m/2} a'_i (\bar{\vw}_i - \bar{\vw}_{i0})^\top \vx \sigma'(\bar{\vw}_{i0}^\top \vx)\right)\\
    +& \frac{1}{\sqrt{m}}\Big( \sum_{i=1}^{m/2} a_i \int_0^1\vx \bracket{\sigma'((t\vw_i + (1-t)\vw_{i0})^\top\vx) - \sigma'(\vw_{i0}^\top \vx)}dt \\
    +& \sum_{i=1}^{m/2} a'_i \int_0^1\vx \bracket{\sigma'((t\bar{\vw_i}+(1-t)\bar{\vw_{i0}})^\top\vx) - \sigma'(\bar{\vw}_{i0}^\top \vx)}dt\Big),
    %f(\vW, \vx) &\approx f(\vW_0, \vx) + (\vW - \vW_0)^\top \nabla_\vW f(\vW_0, \vx) + (\vW - \vW_0)^\top \nabla^2_\vW f(\vW_1, \vx)(\vW - \vW_0)\nonumber\\
    %f(\vW,\vW_0)&\approx  \underbrace{\frac{1}{\sqrt{2m}}\left( \sum_{i=1}^m a_i (\vw_i - \vw_{i0})^\top \vx \sigma'(\vw_{i0}^\top \vx) + \sum_{i=1}^m a'_i (\bar{\vw}_i - \bar{\vw}_{i0})^\top \vx \sigma'(\bar{\vw}_{i0}^\top \vx)\right)}_{(i)},\nonumber
\end{align*}

Then set $\vw_i = \vw_{i0} + \sqrt{\frac{m}{4}}c_i, \bar{\vw}_i = -\sqrt{\frac{m}{4}}c_i + \bar{\vw}_{i0}$, we have
\begin{align*}
    &\norm{\vw_r - \vw_{r0}}_2 \leq \frac{\rfnorm{g_2}}{\sqrt{4m}},\\
    \text{and } &\frac{1}{\sqrt{m}}\left( \sum_{i=1}^{m/2} a_i (\vw_i - \vw_{i0})^\top \vx \sigma'(\vw_{i0}^\top \vx) + \sum_{i=1}^{m/2} a'_i (\bar{\vw}_i - \bar{\vw}_{i0})^\top \vx \sigma'(\bar{\vw}_{i0}^\top \vx)\right)\\
    =& \frac{1}{\sqrt{m}}\left( \sum_{i=1}^{m/2} a_i \sqrt{\frac{m}{4}}c_i^\top \vx \sigma'(\vw_{i0}^\top \vx) - \sum_{i=1}^{m/2} a'_i \sqrt{\frac{m}{4}}c_i^\top \vx \sigma'(\bar{\vw}_{i0}^\top \vx)\right)\\
    =& \sum_{i=1}^m c_i^\top \vx\sigma'(\vw_i^\top \vx)\\
    =& g_3,
\end{align*}
So
\begin{align*}
    \norm{f(\vW, x) - g_3}_{\infty,\cS} =& \Big\|\frac{1}{\sqrt{m}}\Big( \sum_{i=1}^{m/2} a_i \int_0^1\vx \bracket{\sigma'((t\vw_i + (1-t)\vw_{i0})^\top\vx) - \sigma'(\vw_{i0}^\top \vx)}dt \\
    +& \sum_{i=1}^{m/2} a'_i \int_0^1\vx \bracket{\sigma'((t\bar{\vw_i}+(1-t)\bar{\vw_{i0}})^\top\vx) - \sigma'(\bar{\vw}_{i0}^\top \vx)}dt\Big)\Big\|_{\infty,\cS}\\
    \leq&\frac{1}{\sqrt{m}}\norm{m\times C\abs{(t\vw_i + (1-t)\vw_{i0})^\top\vx - \vw_{i0}^\top\vx}\norm{\vx}\norm{\vw_i - \vw_{i0}}}_ {\infty,\cS}\\
    \leq&\frac{C\rfnorm{g_2}^2}{4\sqrt{m}},
\end{align*}
and therefore
\begin{align}
    \label{equ:nn_rf}
    \norm{f(\vW, x) - g_2}_{\infty,\cS}\leq\frac{C\rfnorm{g_2}^2}{4\sqrt{m}} + \frac{C\rfnorm{g_2}}{\sqrt{m/2}}\left(2\sqrt{d} + \sqrt{2\log(1/\delta)}\right).
\end{align}

Finally, set $m$ to be large enough ($= \Omega\bracket{\frac{\rfnorm{g_2}^4}{\epsilon^2}}$) so that the left hand in Equation~\eqref{equ:nn_rf} no more than $\bar{\epsilon}$ and let $R_{\cD, \cB, \epsilon}$ to be $\rfnorm{g_2}/2$. Then
\begin{align*}
    L_*(\vW) =& \frac{1}{n}\sum_{i=1}^n\sup_{\vx\in\cB(\vx_i)} \ell\bracket{f(\vW, \vx), y_i}\\
    \leq&\sup_{i\in[n], \vx\in\cB(\vx_i)}\ell\bracket{f(\vW, \vx), y_i}\\
    \leq& L\sup_{i\in[n], \vx\in\cB(\vx_i)}\bracket{\abs{f(\vW, \vx) - g_2(\vx)} + \abs{g_2(\vx) - g_1(\vx)} + \abs{g_1(\vx) - y_i}}\\
    \leq& 3L\bar{\epsilon}\\
    =&\epsilon.
\end{align*}
The theorem follows by setting $\delta = 0.01$.
\end{proof}

\subsection{Example of Using Quadratic ReLU Activation}\label{sec:quadrelu}
Theorem~\ref{thm:robust_exist} shows that when Assumption~\ref{asmpt:robust_exist_ntk}, \ref{asmpt:smoothness}, \ref{asmpt:activation_additional} hold, we can indeed find a classifier of low robust loss within a neighborhood of the initialization. However, these assumptions are either for generality or simplicity, and for specific activation functions, we can remove these assumptions. As a guide example, we consider the \emph{quadratic ReLU} function $\sigma(x) = \ReLU{x}^2 = x^2\cdot 1_{x \ge 0}$ and its induced NTK. Following the previous work~\cite{bach2017equivalence, bach2017breaking}, we consider the initialization of each $\bfw_r$ with the \emph{uniform distribution on the surface of the sphere of radius $\sqrt{d}$} in this section.\footnote{It is not hard to see that Theorem~\ref{thm:converge_two_layer_projless} still holds under this situation by the same proof.} We can verify that this induced kernel is universal and quantitatively derive the dependency of $\epsilon$ for $R_{\cD, \cB, \epsilon}$ and $m$ in Theorem \ref{thm:robust_exist} for this two-layer network. In order to do so, we need to make a mild assumption of the dataset:\footnote{Our assumption on the dataset essentially requires $\vx_i \neq \pm \vx_j$ since the quadratic ReLU NTK kernel only contains even functions. However, this can be enforced via a lifting trick: let $\tilde \vx =[\vx, 1] \in \mathbb{R}^{d+1}$ , then the data $\tilde \vx$ lie on the positive hemisphere. On the lifted space, even functions can separate any datapoints.}
\begin{asmp}[Non-overlapping]\label{asmpt:no_overlap}
The dataset $\{\vx_i, y_i\}_{i=1}^n\subset\cS$ and the perturbation set function $\cB$ satisfies the following: There does not exist $\vx, \bar{\vx}$ and $i,j$ such that $\vx\in\overline{\cB(\vx_i)}\cup(-\overline{\cB(\vx_i)}), \bar{\vx}\in\overline{\cB(\vx_j)}\cup(-\overline{\cB(\vx_j)})$ but $y_i\neq y_j$.
\end{asmp}
And then we can derive the finite-sum approximation result by random features.
\begin{thm}[Approximation by Finite Sum]\label{thm:qurelu_approx}
For a given Lipschitz function $h\in\cH(\ntk)$. For $\epsilon> 0, \delta \in (0,1)$, let $\vw_1, \cdots, \vw_M$ be sampled i.i.d. from the uniform distribution on the surface of the sphere of radius $\sqrt{d}$ where
\begin{align}
    \label{equ:bound}
    M = \Omega\bracket{C_{\mathcal{D},\cB}\frac{1}{\epsilon^{d+1}}\log\frac{1}{\epsilon^{d+1}\delta}}.
\end{align}
and $C_{\mathcal{D},\cB}, C'_{\mathcal{D},\cB}$ is a constant that only depends on the dataset $\mathcal{D}$ and the compatible perturbation $\cB$. Then with probability at least $1-\delta$, there exists $c_1, \cdots, c_M$ where $c_i\in\R^d$ such that $\hat{h} = \sum_{r=1}^M c_r^\top \vx\ReLU{\vw_r^\top\vx}$ satisfies
\begin{align}
    \label{equ:relu_rf_approx}
    \sum_{r=1}^M \norm{c_r}_2^2 &= O\bracket{\frac{C'_{\mathcal{D},\cB}}{M}},\\
    \norm{h - \hat{h}}_{\infty, \cS} &\leq\epsilon.
\end{align}
\end{thm}
To prove this theorem, we use the $\ell_2$ approximation result in \cite{bach2017equivalence} and translate it to an $\ell_\infty$ approximation result by using Lipshitz continuity. We first state Proposition 1 in \cite{bach2017equivalence}.
\begin{lem}[Approximation of unit ball of $\cH(\ntk)$, Corollary of Proposition 1 in \cite{bach2017equivalence}]\label{lem:bach_prop1}
Let $h\in\cH(\ntk)$. For $\epsilon>0$, let $d\rho$ be the uniform distribution on $\cS$. Let $\vw_1, \cdots, \vw_M$ be sampled i.i.d. from uniform distribution on the surface of the sphere of radius $\sqrt{d}$, then for any $\delta\in (0,1)$, if
\begin{align*}
    M = \exp(\Omega(d))\frac{\rkhs{h}^2}{\epsilon}\log\bracket{\frac{\rkhs{h}^2}{\epsilon\delta}},
\end{align*}
with probability at least $1 - \delta$, there exists $c_1, \cdots, c_M\in\R^d$ such that $\hat{h} = \sum_{r=1}^M c_r^\top \vx\ReLU{\vw_r^\top\vx}$ satisfies
\begin{align}
    \label{equ:bach_l2}
    \sum_{r=1}^M \norm{c_r}_2^2 &= \frac{\rkhs{h}^2\exp(O(d))}{M},\\
    \norm{h - \hat{h}}_{L_2(d\rho)}^2 &= \int_{\cS} \bracket{h - \hat{h}}^2 d\rho \leq \epsilon.
\end{align}
\end{lem}
Then we can give the proof of Theorem~\ref{thm:qurelu_approx}.
\begin{proof}

%We first show that $\hat{h}$ is $2\rkhs{h}$-Lipschitz with high probability. We will use a well-known concentration inequality of the norm of Gaussian vector which state as follow:

%\begin{fact}\label{fact:gaussian_norm}
%Let $z\in\R^{md}$ be drawn from a centered spherical Gaussian, i.e. $z\sim\cN(0, \sigma^2 I)$ where $\sigma>0$. Then we have $\bbP [\norm{z}_2 \geq\sigma\sqrt{md} + t] \leq e^{-t^2/(2\sigma^2)}$.
%\end{fact}
%\begin{proof}
%We refer the reader to Example 5.7 in \cite{boucheron2004concentration} for a reference of this standard result. Combined with $\E[\norm{z}_2]\leq\sigma\sqrt{md}$, which is obtained from Jensen’s Inequality, the aforementioned example gives the desired upper tail bound.
%\end{proof}

Let $Lip(f)$ denote the Lipschitz coefficient of $f$. We consider $\hat{h}$ in Lemma~\ref{lem:bach_prop1}, by the property of Lipschitz coefficient, we have
\begin{align*}
    Lip(\hat{h}) =& Lip\bracket{\sum_{r=1}^M c_r^\top \vx\ReLU{\vw_r^\top\vx}}\\
    \leq& \sum_{r=1}^M Lip\bracket{c_r^\top \vx\ReLU{\vw_r^\top\vx}}\\
    \leq& \sum_{r=1}^M \norm{c_r}_2\norm{\vx}_2Lip\bracket{\ReLU{\vw_r^\top\vx}}\\
    \leq& \sum_{r=1}^M \norm{c_r}_2\norm{\vw_r}_2\\
    \leq& \sqrt{\bracket{\sum_{r=1}^M \norm{c_r}_2^2}\bracket{\sum_{r=1}^M \norm{\vw_r}_2^2}},
    %\leq& \sqrt{\frac{4\rkhs{h}^2}{M}\bracket{\sum_{r=1}^M \norm{\vw_r}_2^2}}\\
    %\leq& 2\rkhs{h} \bracket{1 + O\bracket{\sqrt{\frac{\log\frac{1}{\delta}}{M}}}},\text{\quad (With probability $1 - \delta/2$ using similar argument in Lemma~\ref{lem:gaussian_init})}
\end{align*}
So
\begin{align*}
    Lip(\hat{h}) = \rkhs{h}\exp(O(d)),
\end{align*}
which means $\hat{h}$ has finite Lipschitz coefficient and therefore so does $h - \hat{h}$, and the upper bound of Lipschitz constant $c_L$ only depends on the data and the perturbation. Then we can bound the $\ell_\infty$ approximation error. Suppose for some $\vx\in\cS$, $\abs{h(\vx) - \hat{h}(\vx)} > \epsilon$, since $h - \hat{h}$ is Lipschitz, it is not hard to see that, when $\epsilon$ is small,
\begin{align}\label{equ:contradiction}
\int_{\cS} \bracket{h - \hat{h}}^2 \gtrsim \frac{\pi^{\frac{d}{2}}\epsilon^{d+1}}{\Gamma(d/2+1)c_L^d}\asymp \frac{\epsilon^{d+1}}{c_L^d}\frac{(2\pi e)^{\frac{d}{2}}}{d^{\frac{d+1}{2}}}.
\end{align}
By Lemma~\ref{lem:bach_prop1}, for some constant $C_{\mathcal{D}, \cB}$, when $M = \Omega\bracket{\frac{C_{\mathcal{D}, \cB}}{\epsilon^{d+1}}\log\frac{1}{\epsilon^{d+1}\delta}}$, Equation~\eqref{equ:contradiction} fails, so $\norm{h - \hat{h}}_{\infty, \cS} \leq\epsilon$ holds and at the same time we have $\sum_{r=1}^M \norm{c_r}_2^2 = O\bracket{\frac{C'_{\mathcal{D},\cB}}{M}}$ for some $C'_{\mathcal{D},\cB}$ that only depends on the data and the perturbation.
\end{proof}

Now, we get a similar but more explicit finite-sum approximation result for quadratic ReLU activation, we are then going to show that the RKHS is rich enough that Assumption~\ref{asmpt:robust_exist_ntk} can be derived. %We notice that the NTK under quadratic ReLU has the following explicit expression:
%\begin{align}
%    \label{equ:relu_ntk}
%    \ntk(\vx, \vy) = \E_{\vw\sim\mathcal{N}(0, \mathbf{I}_d)} \langle \vx\ReLU{\vw^\top\vx}, \vy\ReLU{\vw^\top\vy}\rangle.%\frac{\vx^\top\vy\bracket{\pi - \arccos{\vx^\top\vy}}}{2\pi}.
%\end{align}% We denote $\norm{\cdot}_\cH$ the RKHS norm of $\cH(\ntk)$. The following lemma gives a sufficient condition for the function to be in $\cH(\ntk)$.
We have the following lemma to characterize the capacity of the RKHS.
\begin{lem}[RKHS Contains Smooth Functions, Proposition 2 in \cite{bach2017breaking}, Corollary 6 in \cite{bietti2019inductive}]\label{lem:f_in_rkhs}
Let $f:\cS\to\R$ be an even function such that all $i$-th order derivatives exist and are bounded by $\eta$ for $0\leq i\leq s$, with $s\geq (d+3)/2$. Then $f\in\cH(\ntk)$ with $\norm{f}_\cH\leq C_{d}\eta$ where $C_{d}$ is a constant that only depend on the dimension $d$.
\end{lem}

%Under Assumption~\ref{asmpt:robust_exist_ntk}, one can easily construct a smooth classifier $g$ on $\cS$ such that $g(\vx) =  y_i$ for all $\vx\in\cB(\vx_i)$. By Lemma~\ref{lem:f_in_rkhs}, we have $g\in\cH(\ntk)$ with RKHS norm $\rkhs{g}\leq C_{\cD}$ where $C_{\cD}$ is a constant only depends on dataset $\cD$ and perturbation function. We then approximate $g$ using random feature techniques. The following theorem provides the desired result:

Then, by plugging-in the finite-sum approximation theorem (Theorem~\ref{thm:qurelu_approx}) and the theorem of the capacity of RKHS (Theorem~\ref{lem:f_in_rkhs}) to the proof of Theorem~\ref{thm:exist-robust} and combining with the optimization theorem, we can get an overall theorem for the quadratic-ReLU network which is similar to Corollary~\ref{cor:exist-robust} but with explicit $\epsilon$ dependence:
\begin{cor}[Adversarial Training Finds a Network of Small Robust Train Loss for Quadratic-ReLU Network]\label{cor:exist-robust-qurelu}
Given data set on the unit sphere equipped with a compatible perturbation set function and an associated perturbation function $\cA$, which also takes value on the unit sphere. Suppose Assumption~\ref{asmpt:smoothness-loss}, \ref{asmpt:no_overlap} are satisfied. Let $C_{\mathcal{D}, \cB}''$ be a constant that only depends on the dataset $\mathcal{D}$ and perturbation $\cB$. Then for any $2$-layer quadratic-ReLU network with width $m =\Omega(\frac{C_{\mathcal{D}, \cB}''}{\epsilon^{d+1}}\log\frac{1}{\epsilon})$, if we run gradient descent with stepsize $\alpha=O(\epsilon)$ for $T = \Omega(\frac{\sqrt{m}}{\alpha})$ steps, then with probability $0.99$,
\begin{align}
    \min_{t = 1, \cdots, T}L_{\cA}(\vW_t) \leq \epsilon.
\end{align}
\end{cor}

\section{Proof of Theorem~\ref{thm:vcd}}
\begin{proof}
We prove this theorem by an explicit construction of $\left[\frac{n}{2}\right]\times d$ data points that $\mathcal{F}$ is guaranteed to be able to shatter.
Consider the following data points
\begin{align*}
    \bfx_{i, j} = \mathbf{c}_{i} + \epsilon\mathbf{e}_{j} \textrm{ for } i \in \{1, \cdots, \left[\frac{n}{2}\right]\}, j \in [d],
\end{align*}
where $\mathbf{c}_{i} = (6i\delta, 0, \cdots, 0)^\top \in \mathbb{R}^{d}$, $\epsilon$ is a small constant, and $\mathbf{e}_{j} =  (0, \cdots, 1, \cdots, 0) \in \mathbb{R}^{d}$ is the $j$-th unit vector. For any labeling $y_{i, j} \in \{1, -1\}$, we let $P_i = \{j \in[d]:y_{i,j} = 1\}$, $N_i = \{j\in[d] :y_{i,j} = -1\}$, and let $\#P_i = k_i$. The idea is that for every cluster of points $\{\bfx_{i, j}\}_{j = 1}^{n}$, we use 2 disjoint balls with radius $\delta$ to separate the positive and negative data points. In fact, for every such cluster, if $P_i$ and $N_i$ are both non-empty, the hyperplane
\begin{align}\nonumber
\mathcal{M}_i = \{\bfx: (y_{i, 1}, \cdots, y_{i, d})\cdot(\bfx - \bfc_i) = 0\},
\end{align}
clearly separates the points into $\{\bfx_{i, j} :j \in P_i\}$ and $\{\bfx_{i, j} : j \in N_i\}$. Then we can see easily that there exists a $\gamma_{k_i} > 0$ such that for any $r > \gamma_{k_i}\epsilon$, there exist two Euclidean balls $\mathcal{B}_{r}(\bfx_{i,1}'), \mathcal{B}_{r}(\bfx_{i,2}')$ in $\mathbb{R}^{d}$ with radius $r$, such that they contain the set $\{\bfx_{i, j} :j \in P_i\}$ and $\{\bfx_{i, j} : j \in N_i\}$ respectively, and that $\mathcal{B}_{r}(\bfx_{i,1}')$ and $  \mathcal{B}_{r}(\bfx_{i,2}')$ are also separated by $\mathcal{M}_{i}$. Therefore, as long as we take
\begin{align}\nonumber
\epsilon < \delta\max\left(\frac{1}{\gamma_1}, \cdots, \frac{1}{\gamma_{d-1}}, 1\right),
\end{align}
we can always put $r = \delta$. This also holds in the case that $P_i$ or $N_i$ is empty, where we can simply put one ball centered at $\bfx_{i,1} = \bfc_i$ and put $\cB_r(\bfx_{i,2})$ anywhere far away so that it is disjoint from the other balls. Recall that we have chosen $\norm{\mathbf{c}_i - \mathbf{c}_{i'}}_2 \ge 6\delta$ for $i \neq i'$. Such balls $\cB_r(\bfx_{i,l}): i\in\{1, \cdots, \left[\frac{n}{2}\right]\}, l\in\{1,2\},$ are disjoint since $\epsilon \le \delta$, and $\norm{\bfx_{i,l}' - \mathbf{c}_i}_2 \le 2\delta$ for $l = 1, 2$. In this way, since $\mathcal{F}$ is an $n$-robust interpolation class, we can use the fact that there exists a function $f \in \mathcal{F}$ such that for any $i \in \{1, \cdots, \left[\frac{n}{2}\right]\}$, $f(\bfx) = 1$ for $\bfx \in \mathcal{B}_{r}(\bfx_{i,1})$ and $f(\bfx) = -1$ for $\bfx \in \mathcal{B}_{r}(\bfx_{i,2})$. In this way, $f(\bfx_{i, j}) = y_{i, j}$ holds for all $i, j$. Since the labels $y_{i, j}$ can be picked at will, by the definition of the VC-dimension, we know that the VC-dimension of $\mathcal{F}$ is always at least $\left[\frac{n}{2}\right]\times d$.
\end{proof}

% \newpage
% \input{Appendix.tex}
\end{document}